\documentclass{article}
\usepackage[margin=1.2in]{geometry}
\usepackage[affil-it]{authblk} 
\usepackage{latexsym, amsthm, amscd, amsfonts, mathrsfs, amsmath, amssymb, stmaryrd, tikz-cd, mathrsfs, bbm, url, esint, mathtools, enumerate, caption, subcaption}
\usepackage{dsfont}
\usepackage{blkarray}
\usepackage[utf8]{inputenc} 
\usepackage[T1]{fontenc}    
\usepackage{hyperref}       
\usepackage{url}            
\usepackage{booktabs}       
\usepackage{nicefrac}       
\usepackage{microtype}      
\usepackage{xcolor}  
\usepackage{algorithm}
\usepackage{algpseudocode}
\usepackage{enumitem}
\usepackage{comment}

\newcommand{\tr}{\mathrm{Tr}}
\newcommand{\He}{\mathrm{He}}

\newcommand{\Cov}{\mathrm{cov}}

\newcommand{\pr}{\mathbb{P}}
\newcommand{\E}{\mathbb{E}}

\DeclareMathOperator{\Var}{Var}

\DeclareMathOperator{\argmin}{argmin}

\DeclareMathOperator{\Maj}{Maj}

\DeclareMathOperator{\ReLU}{ReLU}
\DeclareMathOperator{\Unif}{Unif}
\DeclareMathOperator{\Rad}{Rad}

\DeclareMathOperator{\sign}{sgn}

\DeclareMathOperator{\NN}{NN}

\DeclareMathOperator{\poly}{poly}

\newcommand{\cD}{\mathcal{D}}
\newcommand{\cN}{\mathcal{N}}
\newcommand{\cE}{\mathcal{E}}
\newcommand{\cP}{\mathcal{P}}

\newcommand{\cI}{\mathcal{I}}
\newcommand{\bR}{\mathbb{R}}

\newcommand{\bN}{\mathbb{N}}

\newcommand{\cC}{\mathcal{C}}
\newcommand{\cL}{\mathcal{L}}
\newcommand{\cB}{\mathcal{B}}

\definecolor{mydarkblue}{rgb}{0,0.08,0.45}
  \hypersetup{ %
    colorlinks=true,
    linkcolor=mydarkblue,
    citecolor=mydarkblue,
    filecolor=mydarkblue,
    urlcolor=mydarkblue,
    }
\definecolor{DSgray}{cmyk}{0,0,0,0.7}
\definecolor{DSred}{cmyk}{0,0.7,0,0.7}

\theoremstyle{plain}
\newtheorem{definition}{Definition}
\theoremstyle{plain}
\newtheorem{ass}{Assumption}
\theoremstyle{plain}
\newtheorem{theorem}{Theorem}
\theoremstyle{plain}
\newtheorem{informaltheorem}{Informal Theorem}
\theoremstyle{plain}
\theoremstyle{plain}
\newtheorem{example}{Example}
\theoremstyle{plain}
\theoremstyle{plain}
\newtheorem{proposition}{Proposition}
\theoremstyle{plain}
\newtheorem{claim}{Claim}
\theoremstyle{plain}
\newtheorem{lemma}{Lemma}
\theoremstyle{plain}
\theoremstyle{plain}
\theoremstyle{remark}
\newtheorem{remark}{Remark}
\theoremstyle{remark}
\theoremstyle{plain}

\title{Learning with Shallow Neural Networks\\ on
Cluster-Structured Features}

\author[]{
Elisabetta Cornacchia}
\author[]{Laurent Massoulié}

\affil[]{\small INRIA, DI/ENS, PSL. }

\date{}

\begin{document}

\maketitle
\begingroup
    \renewcommand\thefootnote{}  
\footnotetext{\small 
\noindent
Email: \texttt{elisabetta.cornacchia@unibocconi.it}, \texttt{laurent.massoulie@inria.fr}.}
\endgroup

\begin{abstract}
  The success of deep learning in high-dimensional settings is often attributed to the presence of low-dimensional structure in real-world data. While standard theoretical models typically assume that this structure lies in the target function—projecting unstructured inputs onto a low-dimensional subspace—data such as images, text or genomic sequences exhibit strong spatial correlations within the input space itself. In this paper, we propose a tractable model to study how these correlations affect the sample complexity of learning with gradient descent on shallow neural networks. Specifically, we consider targets that depend on a small number of latent Boolean variables, and input features grouped into clusters and correlated with the latent variables. Under an identifiability assumption, we show that for a layerwise gradient-descent variant, the sample complexity scales with the number of hidden variables and, when the signal-to-noise ratio is sufficiently high, is independent of the input dimension, up to logarithmic terms. We empirically test our theoretical findings on both synthetic and real data. 
\end{abstract}

\section{Introduction}
A common theme in deep learning is that high-dimensional problems encountered in practice often have an underlying low-dimensional structure, a property widely regarded as central to the success of modern algorithms. Theoretical analyses typically incorporate this assumption by modeling the labels as functions of a low-dimensional projection of the input (e.g., multi-index target functions), while treating the inputs themselves as unstructured—often drawn from an isotropic Gaussian or uniform Boolean distribution~\cite{arous2021online,bietti2022learning,abbe2023sgd}. This simplification makes the analysis tractable, since the training dynamics can be described in terms of a small number of low-dimensional \textit{order parameters} (or \textit{sufficient statistics}) whose evolution can be tracked analytically. Within this framework, several works have shown that the sample complexity to learn with shallow neural networks scales as $O(d^{\cI})$, where $d$ is the input dimension and $\cI$ is an exponent depending on the target function—capturing, for example, the information~\cite{arous2021online,bietti2022learning}, generative~\cite{damian2024computational,dandi2024benefits}, or leap exponent~\cite{abbe2023sgd,abbe2022merged}.

Real-world data, by contrast, rarely resemble such unstructured inputs. 
Images, text, and biological data exhibit strong correlations among input features, so the relevant low-dimensional structure is often encoded in the input distribution itself, not only in the target function. 
Empirical estimates illustrate this phenomenon: for instance, the intrinsic dimension of MNIST, despite its $784$-dimensional pixel representation, is estimated to be on the order of $15$, while that of CIFAR-10 is estimated around $50$~\cite{camastra2003data,costa2004learning,levina2004maximum,facco2017estimating,ansuini2019intrinsic}. 
Several works have studied how correlations among input coordinates can reduce the sample complexity of learning compared with unstructured product distributions~\cite{goldt2020modeling,mousavi2023gradient}. 
However, these analyses typically remain in regimes where the latent dimension diverges, and the resulting complexity still depends critically on the ambient dimension $d$.

In this work we focus on a different regime: the number of latent degrees of freedom remains bounded while the number of observed features grows. 
This is a natural regime for highly redundant data, where many observed coordinates provide noisy views of a small number of underlying factors. 
For example, in genomics, large collections of genes may be co-regulated by a much smaller number of latent biological programs, producing high-dimensional observations with strong feature redundancy (see experiments in Section~\ref{sec:experiments}). 
From a statistical perspective, this is precisely the regime in which one might hope for sample complexity governed by the intrinsic dimension rather than the ambient dimension. 
From an analytical perspective, however, it is not covered by existing high-dimensional asymptotic theories: since the latent dimension is fixed, there is no large-latent-dimensional or thermodynamic limit from which to derive closed dynamical equations or self-averaging order parameters.

Motivated by this setting, we introduce a tractable model of clustered, correlated features. 
Labels are determined by a bounded set of hidden binary \emph{latent variables}, while the high-dimensional \emph{observable variables} are generated from them through a sub-Gaussian noisy process, subject to an essential \emph{identifiability} condition. We prove that a simplified version of gradient descent---with layer-wise training--on a two-layer fully connected network can leverage these correlations to learn any target function of the latent variables. Importantly, in high signal-to-noise (SNR) regimes, the sample complexity becomes independent of the input dimension $d$, up to logarithmic terms.

This perspective is complementary to works showing depth separations for hierarchical data models~\cite{mossel2016deep,cagnetta2024deep,dandi2025computational,ren2026provable,tabanelli2026deep}. 
Rather than asking whether deep networks outperform shallow ones on complex hierarchical structures, we ask a more basic question: when the relevant correlations are already accessible to a shallow architecture, can gradient-based training exploit them without being explicitly given the latent representation? 
Our results give a positive answer in a bounded-latent-dimensional regime, showing that the relevant complexity can be governed by the intrinsic rather than ambient dimension.

\section{Related Work}
\paragraph{Gradient Descent on Product Measures.}
While task learnability with neural networks has been widely studied, most analyses remain within the classical framework where inputs are sampled from an unstructured distribution—typically uniform Boolean or isotropic Gaussian. In this setting, precise characterizations of the sample and time complexities needed to learn single and multi-index models with gradient descent on shallow networks have been obtained in terms of the information~\cite{arous2021online,bietti2022learning}, generative~\cite{damian2024computational,dandi2024benefits,troiani2024fundamental,joshi2025learning}, and leap \cite{abbe2023sgd,abbe2022merged,glasgow2023sgd,kou2024matching,barak2022hidden,joshi2024complexity} exponents of the target function. A few works have shown that shifting the first moment of an isotropic Gaussian (or Boolean) distribution can break symmetries and thereby reduce learning complexity~\cite{malach2021quantifying,cornacchia2023mathematical,daniely2020learning,joshi2024complexity,cornacchia2025low}. \cite{dandi2025computational,tabanelli2026deep} further study the computational benefits of deep versus shallow architectures for learning hierarchical targets, with Gaussian inputs of identity covariance. However, all these works consider product measures and do not address correlations among input features.


\paragraph{Gradient Descent on Structured Inputs.}
Beyond product measures, \cite{mousavi2023gradient,ba2023learning} studied single-index models on Gaussian inputs with spike-covariance structure, showing that when the spike aligns with the target, sample complexity is drastically reduced compared to the unstructured case. \cite{goldt2020modeling} proposed the Hidden Manifold Model (HMM), a generative model in which high-dimensional inputs lie on a low-dimensional manifold and labels depend on their position within it. \cite{bardone2024sliding,szekely2024learning} study detection problems where the signal is encoded into high-order moments of the input, and show separation between neural networks and random feature models. These analyses are in the limit of large input and manifold dimensions, and the complexity still scales with the ambient dimension, as opposed to our analysis which focuses on small latent dimension ($O_d(1)$).
Moreover, \cite{arnal2024learning} empirically study how neural networks learn tasks with hidden factorial structure. 
\cite{abbe2024far} establish a lower bound for weak learning tasks with low globality-degree, where signal resides in global rather than local correlations. 

An active line of work studies models of tree-like correlation structures among features ~\cite{mossel2016deep,poggio2017and}. 
In this direction, \cite{cagnetta2024deep} introduce the Random Hierarchy Model (RHM), where a root label recursively generates hidden variables on a tree and the observed input is given by the leaves. 
They show that deep architectures aligned with the hierarchy can exploit this compositional structure, while shallow networks suffer from the curse of dimensionality. 
More recently, \cite{ren2026provable} proved that layerwise gradient descent on an $L$-layer convolutional network can efficiently learn RHMs, with sample complexity essentially $m^L$, where $m$ is the number of production rules per token and $L$ is the hierarchy depth. 
Subsequent works have explored related hierarchical frameworks in next-token prediction~\cite{cagnetta2024towards}, Transformer architectures~\cite{garnier2024transformers}, and diffusion models~\cite{favero2025compositional,sclocchi2025phase,sclocchi2024probing}. Our work is complementary: we do not study depth separation or architectures matched to a growing hierarchy. 
Instead, we ask whether shallow networks can exploit correlations among observed features in a bounded-latent-dimensional regime. 


A related line of work studies learning in structured Gaussian models in the thermodynamic regime, where the number of samples scales linearly with the input dimension. In particular, \cite{mignacco2020role,mignacco2020dynamical} analyze the learning of mixtures of two Gaussians using a single-layer perceptron, while \cite{refinetti2021classifying} study the Gaussian XOR problem with a two-layer architecture. Furthermore, \cite{loureiro2021learning} extend the teacher–student framework to more realistic data distributions by introducing Gaussian covariate generalization. Our work departs from these approaches in two key respects: (i) we consider a broader class of structured models, beyond purely Gaussian inputs; and (ii) our results extend beyond the thermodynamic regime, with sample complexity scaling at most logarithmically with the input dimension.








\paragraph{Identifiability of Latent Structure.}
Beyond gradient descent learning, some works studied parameters learning in latent variable models, considering settings similar to ours~\cite{allman2009identifiability,anandkumar2014tensor,baik2005phase,benaych2011eigenvalues,arora2012learning,hyvarinen2024identifiability}. In particular, 
\cite{anandkumar2015overcomplete} study latent variable models with topic persistence, where a single topic is assumed to generate a sequence of $n$ consecutive observations—an assumption that resembles our feature model. They provide sufficient conditions for identifiability based on observations of high-order moments of the observations. 
Instead, we focus on supervised learning tasks involving latent representations, using gradient descent on neural networks, rather than just identifiability of the hidden structure. 





\section{Summary of Results}
\begin{figure}[t]
    \centering
    \includegraphics[width=0.5\linewidth]{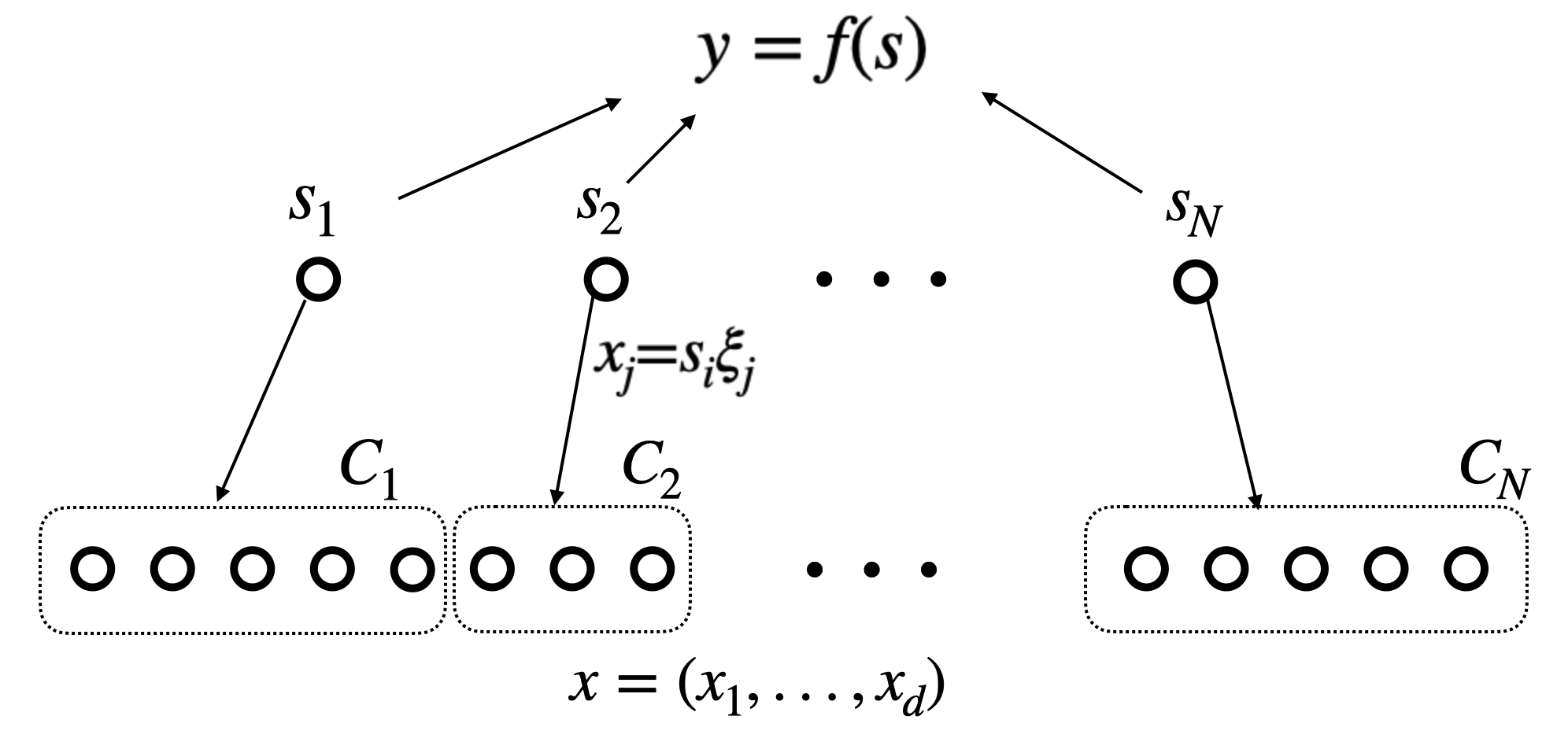}
    \caption{Our clustered feature model, with \textit{latent variables} $s \in \{ \pm 1 \}^N$, \textit{noise variables} $\xi \in \bR^d$, and \textit{observable variables} $x \in \bR^d$.}
    \label{fig:model_features}
\end{figure}
We study a \emph{clustered feature model}, in which the $d$ observable coordinates are partitioned into $N$ disjoint clusters $C_1,\dots,C_N$. Each cluster $C_i$ is associated with a hidden binary \emph{latent (or topic) variable} $s_i\in\{\pm 1\}$, so that all features in the same cluster share a common signal component. Concretely, an input $x\in\bR^d$ is generated by first sampling the latent variables $s=(s_1,\dots,s_N) \in \{ \pm 1 \}^N$, then setting each feature $x_j$ with $j\in C_i$ to  
\(
x_j = s_i \xi_j,
\)
where $\xi_j \in\bR$ is a sub-Gaussian noise variable, independent of $s$. Each $\xi_j$ has mean $m_j := \E[\xi_j]$ (interpreted as the \emph{signal strength}), and is independent across samples. The label is given by an arbitrary function $f$ of the latent variables: $y=f(s)$. See Figure~\ref{fig:model_features} for an illustration of our data model and Def~\ref{def:datamodel} for a formal definition. Throughout, $N$ is treated as fixed while $d$ grows.

This model captures high-dimensional data in which many observed coordinates are noisy, redundant measurements of a much smaller number of latent factors. 
A natural example is genomics, where groups of genes may be co-regulated by a small number of latent biological programs, producing clustered covariance patterns and high feature redundancy. 
We refer to Section~\ref{sec:experiments} for examples of real datasets with structure consistent with this model.

Our main structural assumption ensures that the latent variables are statistically recoverable. 
For each cluster, define
\[
v_i:=\sum_{j\in C_i}m_j^2,
\qquad
v_{\rm sum}:=\sum_{i=1}^N v_i=\|m\|_2^2.
\]
Here $v_i$ is the total signal power carried by cluster $C_i$, while $v_{\rm sum}$ is the total signal power across all observed coordinates. 
We assume that no cluster is vanishingly weak: there exists a constant $c>0$ such that
\[
\frac{v_i}{v_{\rm sum}}\ge c,
\qquad \forall i\in[N].
\]
This condition ensures that each latent variable leaves a detectable footprint in the observed features. 
When the noise variables have bounded variance, $v_{\rm sum}$ naturally plays the role of an aggregate signal-to-noise ratio. 
Our main theorem focuses on the high-SNR regime
\[
v_{\rm sum}=\Theta_d(d),
\]
where the total signal grows linearly with the ambient dimension.


Let us give few examples of simple models captured by this framework. 

\begin{example}[Mixture of Gaussians]  \label{exa:mixture_Gauss}
If the noise variables are Gaussian, our model recovers standard Gaussian mixture examples. 
For instance, let $N=1$, with $s_1\sim\Unif\{\pm1\}$, $y=s_1$, and
\(
\xi\sim\mathcal N(m,\sigma^2 I_d)
\)
for some $m\in\mathbb R^d$ and $\sigma>0$. 
Then $x=s_1\xi$ is distributed as a symmetric mixture of two Gaussians,
\[
x\sim \frac12\mathcal N(m,\sigma^2 I_d)
+\frac12\mathcal N(-m,\sigma^2 I_d),
\]
as studied in~\cite{mignacco2020dynamical,mignacco2020role}. 
Similarly, when $N=2$, with label $y=s_1s_2$, the model gives a Gaussian XOR problem~\cite{refinetti2021classifying}. 
Previous works typically study the low-SNR regime $\|m\|_2=\Theta_d(1)$, where the sample complexity remains linear in the ambient dimension. 
By contrast, our analysis focuses on the high-SNR regime
\(
v_{\rm sum}=\|m\|_2^2=\Theta_d(d),\)
\(\sigma^2=\Theta_d(1).
\)

\end{example}

\begin{example}[Binary Symmetric Cluster (BSC) Model]  \label{exa:treemodel}
A particularly simple instance of our framework is obtained when the observed variables are Boolean noisy copies of the latent variables. 
For each cluster $C_i$ and each coordinate $j\in C_i$, let
\(
\xi_j\sim \Rad(1-\delta),
\) \(\delta\in[0,1],
\)\footnote{Rademacher distribution: $z \sim \Rad(p)$ means $\pr(z=1) = 1-\pr(z=-1) = p$.}
independently across coordinates, so that
\(
x_j=s_i\xi_j, j\in C_i.
\)
Equivalently, each observable coordinate is obtained by passing the latent bit $s_i$ through a binary symmetric channel (BSC) with flip probability $\delta$. 
In this case,
\(
m_j=\mathbb E[\xi_j]=1-2\delta,
\)
and therefore
\(
v_i=|C_i|(1-2\delta)^2,
v_{\rm sum}=d(1-2\delta)^2.
\)
When the clusters have comparable sizes, the identifiability condition is satisfied. When the clusters have the same size, we refer to this setting as the \emph{homogeneous BSC model}.
\end{example}

Our goal is to characterize the number of samples required for a two-layer neural network 
\(
\NN(x;\theta) := \sum_{h=1}^n a_h \, \sigma(w_h^\top x + b_h),
\)
trained by stochastic gradient descent (SGD), to achieve small test error:
\(
\E_{(x,y)}\!\left[ \big(y - \NN(x;\hat\theta)\big)^2 \right]
\). We will consider the latent dimension $N$ and the target function $f$ to be fixed, and we  study the scaling of the complexity as $d$ grows.

Our main theorem reads as follows (see Theorem~\ref{thm:clusters} for a formal statement).
\begin{informaltheorem} \label{thm:clusters_informal}
For any $\epsilon,\delta>0$, under the clustered feature model and assuming identifiability and high-SNR conditions, layerwise-SGD on a two-layer network with $O(\log(1/\epsilon))$ hidden neurons, requires at most $O(\frac{\log(d)^2d}{v_{\rm sum}\delta^2})$ samples and $O(1/\delta^2)$ steps to achieve error at most $\delta$ with probability at least $1-\poly(\epsilon)$.
\end{informaltheorem}
In particular, since the overall signal satisfies $v_{\rm sum}= \theta(d)$, the sample complexity depends on the ambient dimension $d$ only logarithmically. 
This result holds for Gaussian 
initialization and some polynomial activation functions 
(Assumption~\ref{ass:activation}); the formal statement is given in 
Theorem~\ref{thm:clusters}.

We then specialize to the homogeneous BSC model and obtain a result that holds in the wider SNR regime $v_{\rm sum} = \Omega_d(\log(d)^4)$, 
under a mild non-degeneracy condition on the target function $f$ 
(Theorem~\ref{thm:homogeneous_clusters}).

\section{Setting and Formal Result}


\paragraph{Data Model.} Let us give a formal definition of our data model. As we mentioned earlier, we assume $N$ and $f$ to be fixed, and we study the scaling of the sample complexity as $d$ grows.

\begin{definition}[Clustered-Structured Features] \label{def:datamodel}
Let $(x,y) \in \bR^d \times \bR$ and $f: \{\pm 1\}^N \to \bR$.
Assume that:
\begin{itemize}
    \item $s \sim \Unif\{\pm 1\}^N$;
    \item There exists a partition $\cC = (C_1,...,C_N)$ of $[d]$ into $N$ disjoint clusters:
    $[d] = \cup_{i \in [N]} C_i$, such that for all $i \in [N]$ and $j \in C_i$,
    $$x_j = s_i \xi_j,$$
    where $(\xi_j)_{j \in [d]}$ are mutually independent,
    independent of $s$, sub-Gaussian random variables with
    $\|\xi_j\|_{\psi_2} \le \psi = O_d(1)$, finite means
    $m_j := \E[\xi_j] = O_d(1)$, variances $\tau^2_j := \Var(\xi_j)$ such that $ \frac{1}{d} \sum_{j=1}^d \tau_j^2 = \theta_d(1)$,
    and bounded third moments $\sup_j \E|\xi_j - m_j|^3 = O_d(1)$;
    \item $y = f(s)$.
\end{itemize}
We say that $(x,y) \sim \cD_{N,f,\cC,\xi}$.
\end{definition}
We call the $s \in \{\pm 1 \}^N$ the \emph{latent} (or \emph{topic}) variables, the $x \in \bR^d$ the \emph{observable} variables, and the $\xi \in \bR^d$ the \emph{noise} variables. Crucially, we assume that the size of the topic space $N$ is constant, as the observable dimension $d$ grows, i.e.:
    $N = O_d(1).$
Notice that, under this model, the covariance structure of $x$ is block diagonal with respect to the partition $\mathcal C$, with within-cluster correlations induced by the shared latent variable $s_i$ and no cross-cluster correlations.
Throughout the paper, we assume that the following identifiability assumption is satisfied.
\setcounter{ass}{0}
\begin{ass}[\textbf{(A1)}, Identifiability] \label{ass:data}
    For each cluster $C_i$, $i \in [N]$, define: $ v_i := \sum_{j \in C_i} m_j^2  $, $v_{\min} := \min_{i \in [N]} v_i$ and $v_{\rm sum} := \sum_{i \in [N]}v_i = \| m\|_2^2 $, where $m_j:= \E[\xi_j]$, for $j \in [d]$. We assume that there exists a $c>0$ such that 
    \begin{align*}
        a) \quad \frac{v_{\min}}{v_{\rm sum}} \geq c, \qquad  \qquad  b) \quad \frac{v_{\rm sum}}{d} \geq c.
    \end{align*}
\end{ass}
Assumption~\ref{ass:data}a) states identifiability of individual clusters. It prevents degenerate cases in which one cluster contributes vanishingly small signal compared to the others, making the corresponding latent variable statistically unrecoverable. On the other hand, Assumption~\ref{ass:data}b) is a lower bound on the signal-to-noise (SNR) ratio, which is required for our main Theorem (Thm.~\ref{thm:clusters}) to hold. In fact, under our data model assumption $ \sum_{j=1}^d \tau_j^2 = \theta_d(d)$, Assumption~\ref{ass:data}b) is equivalent to $\|m\|_2 / \| \tau\|_2 \geq c'>0$. In Theorem~\ref{thm:homogeneous_clusters} we show that for the specific case of homogeneous BSC models, Assumption~\ref{ass:data}b) can be released to $v_{\rm sum} = \omega_d(\log(d)^2)$.

\paragraph{Architecture.} We adopt standard choices for both architecture and optimizer, since our focus is on the impact of the data structure. Concretely, we use a 2-layer fully connected network: 
\begin{align}
    \NN(x;\theta)=\sum_{h=1}^n a_h \sigma(w_h^T x + b_h),
\end{align} where $\theta = (w,a,b) \in \bR^{d\times n}\times \bR^{n} \times \bR^n$ are the network's trainable parameters. 
In our main Theorem~\ref{thm:clusters}, we assume the activation to be a polynomial, satisfying the following technical assumption.

\begin{ass}[\textbf{(A2)}, Activation]\label{ass:activation}
Let $\sigma(x)=\sum_{\ell=0}^P c_\ell x^\ell$ be a polynomial activation of degree
$P=\Theta_d(1)$, with $P\ge 2^N$. Let
\(
    \mu := \frac1d\sum_{j=1}^d \tau_j^2,
    \) and \(
    v_{\rm sum}:=\|m\|_2^2,
\)
and for $v>0$ define
\[
    S_v(t):=\mathbb E_{G\sim\mathcal N(0,v)}[\sigma'(t+G)].
\]
We assume that there exists $c_0>0$ such that, for all $0\le k\le N$,
\[
    \left|S_{\mu+v_{\rm sum}/d}^{(k)}(0)\right|\ge c_0.
\]
\end{ass}

\begin{remark}
We note that in Assumption~\ref{ass:activation} the Hermite non-degeneracy condition is tied to the data model through the scalars
\(\mu \) (the reference smoothing scale) and \(v_{\rm sum}\) (from Assumption~\ref{ass:data}). This dependence is for convenience of the analysis: it allows us to match the Gaussian smoothing scale to the input signal-to-noise and is precisely what we need to certify linear separability after the first gradient step on the first layer (see the Algorithm description below). We remark that this condition is satisfied by most polynomial activations of degree at least $2^N$ (see Lemma~\ref{lem:quantitative-beta} in Appendix for details). 
\end{remark}


\paragraph{Algorithm.} As training algorithm, we take layerwise stochastic gradient descent (SGD): first we take one gradient step on the first layer (holding the second layer fixed), then we train the second layer until convergence (keeping the first layer frozen). We take the standard square loss. We take a Gaussian initialization for the first layer's weights $w$, with normalized variance, while the second layer's weights are initialized to a small enough value, so that the interaction term in the square loss is negligible with respect to the correlation term (similarly to e.g.~\cite{abbe2023sgd,mousavi2023gradient}). 
The bias weights are set to zero for the first phase, and then sampled uniformly at random in a given interval, and kept fixed during training, for the second phase. This stylized layerwise simplification of gradient descent matches several recent theoretical studies in deep learning~(\cite{damian2022neural,dandi2023two,ba2022high,barak2022hidden,cornacchia2023mathematical}). We consider the online setting, where fresh batches of samples are generated at each training step. Further details on the algorithm can be found in Appendix~\ref{sec:proof_clustered} (Algorithm~\ref{alg:layerwise-sgd}). In our experiments in Section~\ref{sec:experiments} we use standard SGD, with both layers trained jointly.

\paragraph{Main Result.} Let us state our main theorem.
\begin{theorem} \label{thm:clusters}
Let $f:\{ \pm 1 \}^N \to \bR$ be a target function such that $\Var(f) =1$ and let $(x,y) \in \bR^d \times \bR$ be drawn from $\cD_{N,f,\cC,\xi}$, according to Def.~\ref{def:datamodel}.
    Assume A\ref{ass:data} is satisfied with $v_{\min}: = \min_{i \in [N]}\sum_{j \in C_i}m_j^2$ and $ v_{\rm sum} := \| m\|_2^2$. 
    Then, there exist $c_1,c_2>0$ such that for any $\delta,\epsilon \in (0,1/4)$, a two-layer network with activation satisfying A\ref{ass:activation}, $n =\theta(2^N \log(1/\epsilon)) $ hidden neurons and Gaussian initialization with variance $1/d$ for the first layer's weights, trained by layerwise-SGD with the squared loss, batch size $B = \theta(N \log(d \log(1/\epsilon))^2d/v_{\rm sum})$ for $T_1 =1$ steps on the first layer, with learning rate $\gamma_1=\theta(1/\sqrt{v_{\rm sum}})$ and $T_2 = \theta(2^{2N}/\delta^2)$ steps, with $\gamma_2 = \theta(v_{\min}^{-2P}\delta/2^N)$ on the second layer, with probability $1-c_1\epsilon^{c_2}$, over the initialization and the training samples, will output a network $\NN(x;\theta^{T_1+T_2})$ such that
       $ \E_{x,y} \Big| y- \NN(x;\theta^{T_1+T_2})\Big| \leq \delta.$
\end{theorem}
Since Assumption~\ref{ass:data} implies $v_{\rm sum}\gtrsim d$ and $N=O_d(1)$,
for fixed $\delta,\epsilon>0$ Theorem~\ref{thm:clusters} gives
$T_1+T_2=O_d(1)$ and
\(
B=O_d(\log^2 d).
\)
Thus the sample complexity $B(T_1+T_2)$ scales polylogarithmically in $d$.
On the other hand, the time complexity (measured as $ \dim(\theta)\cdot (T_1+T_2) \cdot B$, where $\dim(\theta)$ denotes the number of trainable parameters in the network) depends linearly
on $d$, up to logarithmic factors. 


\begin{remark}
We assume a finite number of clusters $N$ because a two-layer network imposes a
representational bottleneck. In our construction, the first layer learns the
topic variables $s\in\{\pm1\}^N$, and the second layer fits $f(s)$ (see
Sec.~\ref{sec:proof_outline}). Since the space of all functions on
$\{\pm1\}^N$ has dimension $2^N$, our construction uses width of order $2^N$ in
the second layer. Thus, we keep $N$ fixed. Although $N$ could scale as
$\log d$, we do not analyze that regime here. It is natural to ask whether
adding depth, for instance using three layers, could reduce the exponential
dependence on $N$ to a polynomial one, allowing $N$ to grow with the input
dimension; we leave this direction to future work. Finally, the assumption
$s\sim\mathrm{Unif}\{\pm1\}^N$ is made for simplicity of exposition, and we
expect that the analysis can be extended to more general zero-mean
distributions.
\end{remark}


\begin{remark}
Crucially, we exploit the random Gaussian initialization of the first layer to
induce diversity among hidden units. The required non-degeneracy follows from
the Carbery--Wright anti-concentration inequality for polynomials under
log-concave measures; see Lemma~\ref{lem:non_degeneracy} and
\cite{carbery2001distributional}. This motivates our use of polynomial
activations and Assumption~\ref{ass:activation}. Gaussian initialization also
allows us to control the second moments of the resulting random polynomial
features via Hermite expansions. We expect that similar arguments extend to
other log-concave initializations. Finally, analogous extensions to common
non-polynomial activations, such as $\mathrm{ReLU}$, may be possible by
truncating their Hermite expansions, at the cost of additional approximation
and tail-control arguments.
\end{remark}

In the next section we give a proof outline for Theorem~\ref{thm:clusters}. We refer to Appendix~\ref{sec:proof_clustered} for the complete proof.

\subsection{Proof Outline for Theorem~\ref{thm:clusters}}
\label{sec:proof_outline}
The training proceeds in two phases. With a small second-layer scale at initialization, a single layerwise-SGD step on the first layer moves each weight toward
a random linear combination of the cluster mean directions, so each hidden unit effectively computes a one-dimensional \emph{projection} of the input onto a latent direction $u(s)=s^\top\tilde\alpha(w^0)$, where $\tilde \alpha(w^0)$ are random variables that depend on the first layer's initialization. In particular, using a Berry-Esseen type bound (Lemma~\ref{lem:pop_grad}), we show that:
\begin{align} \label{eq:alpha_main}
(\tilde\alpha(w^0))_i=\frac{v_i}{v_{\rm sum}}\;\E_{s}\!\big[f(s)\,s_i\,\cP_s(G_{w^0})\big],\quad i\in[N],
\end{align}
where $v_i:=\sum_{j\in C_i}m_j^2$, $v_{\rm sum}:=\|m\|_2^2=\sum_{i\in[N]}v_i$, $G_{w^0}\sim\mathcal N(0,v_{\rm sum}/d)$, and $\cP_s$ are polynomials in $G$ of degree $\deg(\sigma)-1$ determined by the activation, that depend on $s \in \{ \pm 1 \}^N$. 

To conclude that the first phase has produced useful features, we need to show
that the target function can be expressed using the hidden units obtained after
this update. Since these hidden units depend on the latent variables only through
the scalar projection $u(s)$, this reduces to showing that the projection retains
all information relevant to the label. In other words, we need $f$ to be
projection-consistent: for all $s,t\in\{\pm1\}^N$, $f(s)\neq f(t)$ implies
$u(s)\neq u(t)$. Crucially, we exploit the random Gaussian initialization, which induces diversity
across hidden units. Using the Carbery--Wright anti-concentration inequality
(Lemma~\ref{lem:non_degeneracy}), we show that the projected values
$\{u(s):s\in\{\pm1\}^N\}$ are well separated with high probability over the
initialization. Assumption~\ref{ass:data} ensures that all latent coordinates
$i\in[N]$ contribute non-negligible signal, while the condition
$v_{\rm sum}/d=\Theta_d(1)$ keeps the Gaussian argument $G_{w^0}$ at a
nondegenerate variance scale, which is needed for the anti-concentration
argument.

Then, by sampling sufficiently many hidden-unit biases uniformly over a wide
interval, an interval-hitting argument ensures that at least one bias lies near
each distinct projection value. The resulting hidden responses form a
well-conditioned system. Solving this system yields a \emph{certificate}, i.e.,
a choice of second-layer weights that exactly interpolates $f(s)$ on the
projection grid (Lemma~\ref{lem:certificate}). In the second phase, we freeze
the first layer and fit the output weights by SGD on the square loss; standard
convergence results for convex objectives then drive the error to the
target level (Lemma~\ref{lem:generalization}).


\section{Special Case: Homogeneous Clusters}
Let us consider for simplicity the homogeneous BSC model, where clusters have equal size (see Example~\ref{exa:treemodel}). In particular, let $\xi_j \overset{iid}{\sim} \Rad(1-\delta)$, for some $\delta \in (0,1)$, for all $j \in [d]$, and $|C_i|=k =d/N$, for all $i \in [N]$, where for simplicity we assume that $d$ is a multiple of $N$.
Note that in this case $v_{\rm sum} = d (1-2\delta)^2$ and $v_{\min}/v_{\rm sum} =1/N$.  

If $|\delta-\frac 12 | = \theta_d(1)$, and therefore $v_{\rm sum} = \theta_d(d)$, then Theorem~\ref{thm:clusters} applies, and for fixed $\epsilon>0$, the number of samples needed to achieve $\epsilon$ error scales only logarithmically in $d$.

Specializing to the homogeneous BSC model, allows us to obtain a sharper guarantee, valid as soon as $v_{\rm sum} =\Omega_d(\log(d)^4)$, rather than the high-SNR condition $v_{\rm sum} = \theta(d)$ required by our general result. To this end, we consider a deterministic initialization of $w_j^0 = 1/\sqrt{d}$, for all $j \in [d]$, and we focus on the simpler $\ReLU$ activation function. Because of the deterministic initialization, we cannot exploit the randomness of the initialization as in eq.~\eqref{eq:alpha_main} and Lemma~\ref{lem:non_degeneracy} to show that the target function is projection consistent. Thus, we impose the following non-degeneracy assumption on the target $f$.
\begin{ass}[Majority Margin]
\label{ass:nondegeneracy_f}
    For $T \subseteq [N]$ and $i \in [N]$ let 
    \begin{align} \label{eq:cTi}
        c_{T,i} := \left\{
                \begin{array} {ll}
                 1 \quad &\text{ if  } \quad T =\{ i \} , \\
                \hat \Maj_N(|T|-1) \quad &\text{ if }\quad  \{i \} \subset T,\\
                \hat \Maj_N(|T|+1) \quad &\text{ if }\quad  \{i \} \not\subset T,
                \end{array} \right.
    \end{align}
    where $\hat \Maj_N(k)$ are the Fourier-Walsh coefficients of the Majority function in dimension $N$~\cite{o'donnell_2014}, see Appendix~\ref{app:majority_Fourier}. We say that $f$ has $\Delta$-\emph{margin}, for $\Delta>0$, if for all $s,t \in \{ \pm 1 \}^N$ such that $f(s) \neq f(t)$,
    \begin{align} \label{eq:assumption}
         \Big|\sum_{i=1}^N (s_i-t_i) \sum_{T \subseteq [N]} \hat f(T) c_{T,i}  \Big| \geq \Delta.
    \end{align}
\end{ass}
\noindent 
Our Theorem reads as follows.
\begin{theorem}
\label{thm:homogeneous_clusters}
    Assume data are sampled from the homogeneous BSC model with $N =O_d(1)$ clusters, noise parameter $\delta \in [0,1/2]$ such that $|\delta-1/2| =\Omega(\log(d)^2/\sqrt{d})$ and target function $f$ with $\Delta$-margin (according to Assumption~\ref{ass:nondegeneracy_f}). Then, there exists a constant $C>0$ such that for all $\epsilon>0$, layerwise-SGD with the squared loss, batch size $B = O(\frac{\log(d)^2}{(1-2\delta)^2})$ on a two-layer network with $n =\Omega(2^N\log(d)/\Delta) $ hidden neurons, after $T_1 =1, T_2 = O(\log(d)^2 2^N\epsilon^{-2}\Delta^{-4})$ steps with probability $1-d^{-C}$ will output a network $\NN$ such that
        $\E_{(x,y)} \left[(y- \NN(x))^2 \right] \leq \epsilon.$
\end{theorem}
Thus, at the threshold $|\delta-1/2|=\Omega(\log d^2/\sqrt d)$, equivalently
$v_{\rm sum}=\Omega_d(\log^4 d)$, the theorem still guarantees learning with
polynomial-in-$d$ sample size.
In Appendix~\ref{sec:clustering_algo} we compare the sample complexity guarantee in Theorem~\ref{thm:homogeneous_clusters} with the performance of classical clustering methods.

\section{Experiments} \label{sec:experiments}
\label{sec:hierarchical}

\begin{figure*}[t]
    \centering
    \includegraphics[width=0.49\linewidth]{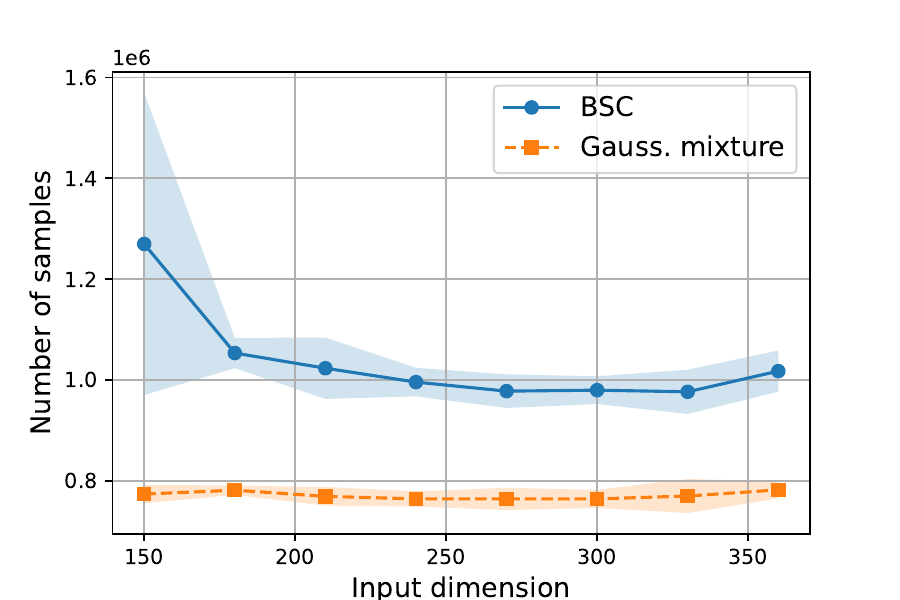}
    \includegraphics[width=0.49\linewidth]{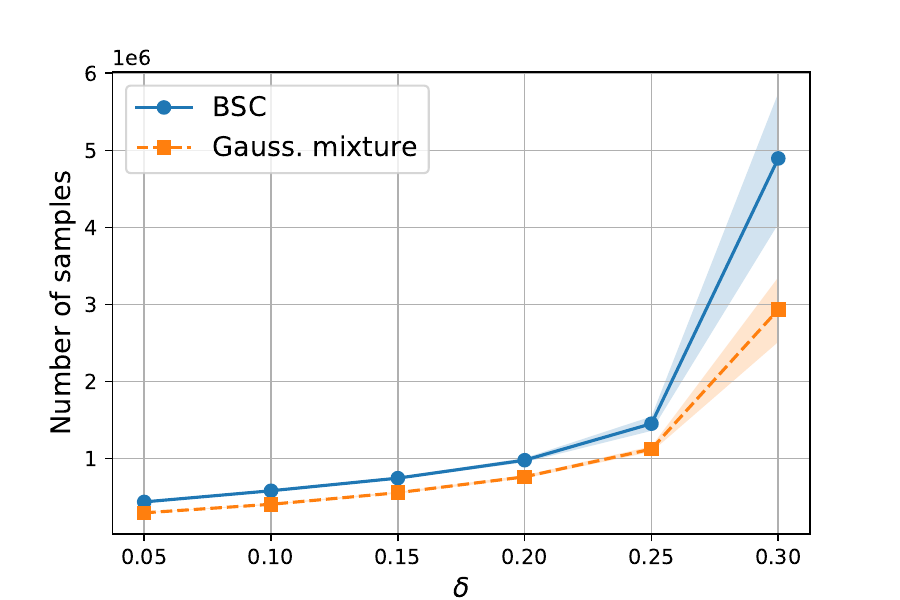}
    \caption{Learning parity functions of $N=3$ latent bits, under the homogeneous BSC model (blue) and the Gaussian mixture model (orange), using a two-layer ReLU network with $1000$ hidden neurons, trained by SGD (square loss, batch size $64$, learning rate $0.001$). We plot the number of fresh samples required to achieve test loss below $0.05$: (Left) as a function of the input dimension $d$ for $\delta=0.2$, and (Right) as a function of $\delta$ for $d=300$.}
    \label{fig:clusters}
\end{figure*}

\paragraph{Experiments on Synthetic Data.}
We empirically verify whether our theoretical results hold in practice, for standard SGD training (both layers trained jointly). We consider learning parity targets: for latent dimension $N =3$, we sample $s \sim \Unif \{\pm 1 \}^N$ and set $y = \prod_{i=1}^N s_i$. Parities under uniform inputs are notoriously hard for neural networks: they are uncorrelated with every strict subset of coordinates, so successful learning requires identifying all relevant hidden variables. This makes them a natural choice for testing whether a method can recover the latent structure in our setting.
In all experiments, we train a two-layer $\ReLU$ network with $1000$ hidden neurons and standard uniform initialization, using SGD (batch size $64$) and generating fresh samples at each step. Training stops when the in-distribution test error falls below $0.05$, and each experiment is repeated $5$ times. We report mean and $95 \%$ confidence intervals.

We consider two data-generating models.
\textbf{(i)} The \emph{BSC homogeneous cluster} (Example~\ref{exa:treemodel}), with noise parameter
\(
\delta \in \{0.05, 0.1, 0.15, 0.2, 0.25, 0.3\}.
\)
\textbf{(ii)} A \emph{Gaussian mixture model}, where the noise variables satisfy
\(
\xi_j \overset{iid}{\sim} \mathcal{N}(m,\sigma^2)
\), for $j \in [d]$,
with mean \(m = 1-2\delta\) and variance \(\sigma^2 = 4\delta(1-\delta)\), using the same values of \(\delta\) as the tree model.
These choices ensure that the two models are directly comparable, as their noise variables have matching first and second moments. In the \emph{left panel}, we plot the number of fresh samples required for learning as a function of the input dimension \(d\), fixing \(\delta=0.2\). We observe that, for sufficiently large \(d\), the sample complexity remains essentially constant in \(d\), in agreement with our theoretical predictions (up to logarithmic factors). In the \emph{right panel}, we plot the number of samples required for learning as a function of \(\delta\), for a fixed input dimension \(d=300\). As \(\delta\) approaches \(1/2\), learning becomes progressively less efficient, again consistently with the theoretical analysis.

\paragraph{Experiments on Real Data.}
We consider the public single-cell RNA-seq dataset GSE96583, hosted on GEO,
using batch-2 cells from control (unstimulated) samples~\cite{kang2018multiplexed}.
Each observation is a cell, and each feature is the measured expression level of one gene;
cells are annotated with their type.
A clustered feature structure is expected because the ambient dimension is large yet
many genes are redundant: multiple genes act as correlated noisy readouts of a smaller
number of latent biological programs, such as cell state.

As a supervised learning task we consider multiclass prediction of cell type from
gene expression.
We exclude rare cell types (Dendritic cells, Megakaryocytes) and merge
CD4~T, CD8~T, and NK cells into a single T/NK class, yielding a three-class problem:
B cells, Monocytes, and T/NK.
After standard quality-control filtering (minimum number of expressing cells per gene), approximately $11{,}990$ cells
and $10{,}500$ genes remain.
Figure~\ref{fig:RNAdata}(left) shows the t-SNE visualization of the three classes,
which are well separated in gene-expression space. We select the top-$d$ most variable genes by dispersion (variance/mean),
using no class-label information, and compute the gene-gene Pearson correlation matrix
across all cells.
Reordering genes by hierarchical clustering reveals a block structure,
shown in Figure~\ref{fig:RNAdata}(center).
We estimate $\mathrm{SNR} \approx \widehat{v}_{\mathrm{sum}} /
\mathrm{Tr}(\widehat{\Sigma}) \approx 0.29$, where $\widehat{\Sigma}$ denotes
the empirical covariance matrix of the selected features. We then train a two-layer $\mathrm{ReLU}$ network with $128$ hidden units to
predict the three cell types; implementation details are given in
Appendix~\ref{app:experiments_details}.
Figure~\ref{fig:RNAdata}(right) shows test accuracy as a function of training
size $n$ for varying input dimension $d$: curves for $d \geq 200$ nearly coincide,
consistent with a sample complexity that is essentially independent of $d$.



\begin{figure*}[t]
\centering

    \centering
    \includegraphics[width=0.31\linewidth]{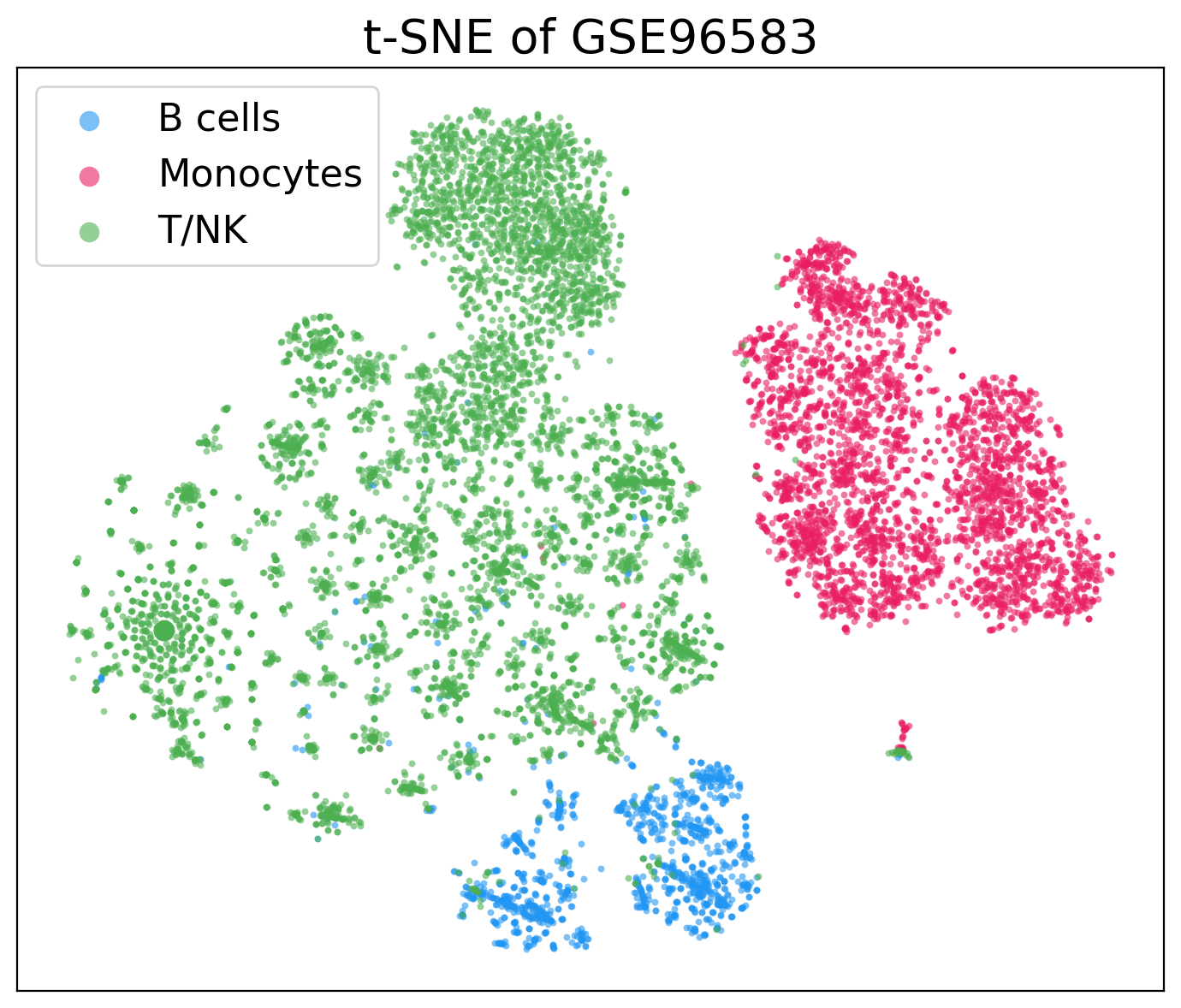}
    \includegraphics[width=0.32\linewidth]{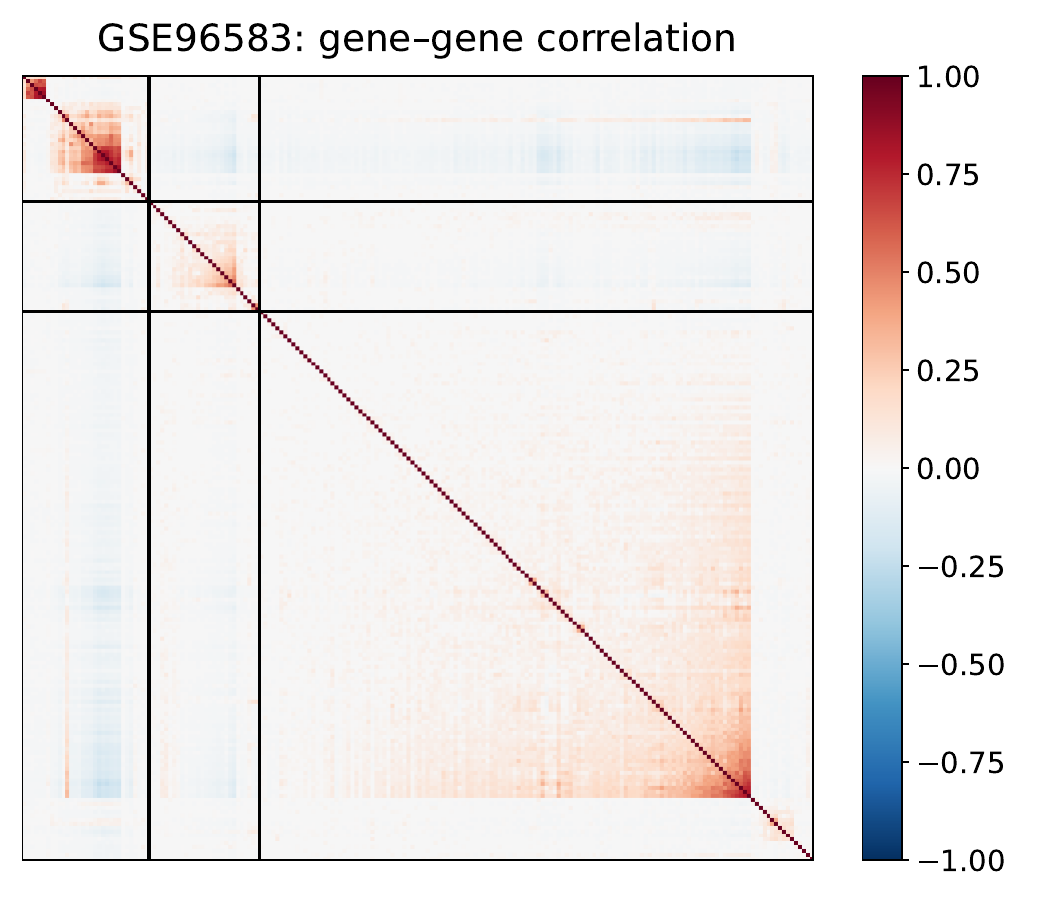}
    \includegraphics[width=0.35\linewidth]{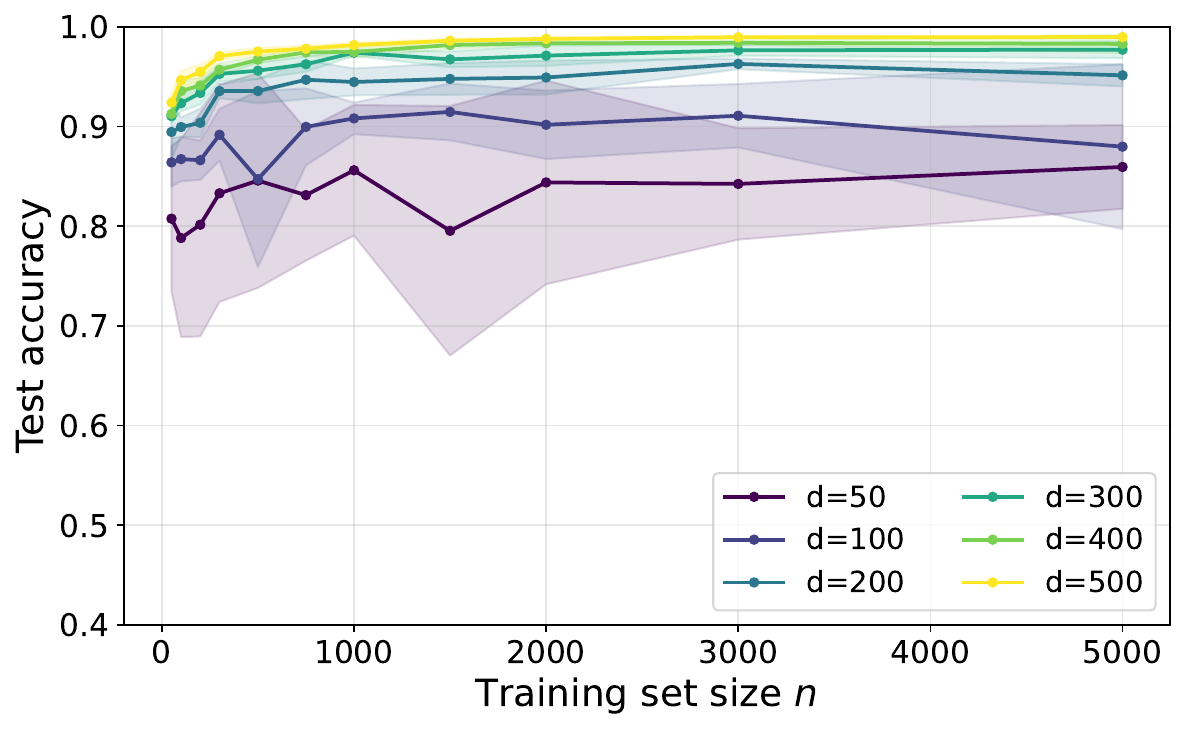}

\caption{Cell-type classification on the single-cell RNA-seq dataset GSE96583. (left) t-SNE visualization of the three cell types (B cells, Monocytes, T/NK). (center) Gene-gene Pearson correlation matrix, with genes reordered by hierarchical clustering; diagonal blocks reveal a latent cluster structure. (right) Test accuracy as a function of training size $n$ for varying input dimension $d$; curves for $d \geq 200$ nearly coincide, consistent with dimension-independent sample complexity.}
\label{fig:RNAdata}
\end{figure*}

\section{Conclusion}
In this paper, we introduce a tractable model of correlated, clustered inputs with a small number of latent binary variables and proved that, under identifiability and technical conditions, layerwise SGD on a two-layer network can learn any latent target with a sample complexity that is independent of the ambient dimension, up to logarithmic terms. Our experiments with standard SGD confirm these results. Looking forward, several natural extensions arise. First, deeper architectures may help overcome the exponential dependence on the number of latent variables and improve computational efficiency. Second, it would be interesting to extend the framework to larger alphabets and richer channels: we expect similar guarantees for finite alphabets, while handling continuous latent variables will likely require new technical tools. Third, one could allow cluster memberships to vary across samples.


\section*{Acknowledgement}
EC was supported by the French government under management of Agence Nationale de la Recherche as part of the “Investissements d’avenir” program, reference ANR19-P3IA-0001 (PRAIRIE 3IA Institute). LM was supported by the PR[AI]RIE-PSAI – Paris School of Artificial Intelligence, reference: ANR-23-IACL-0008.

\bibliography{references}
\bibliographystyle{alpha}


\newpage 

\appendix

\section{Proof of Theorem~\ref{thm:clusters}}
\label{sec:proof_clustered}


\begin{algorithm}[H]
\caption{\textit{Layerwise SGD with Gaussian initialization}\\ Init. scale $\kappa$, bias range $A$, learning rates $\gamma_1, \gamma_2$, step counts $T_1=1, T_2$, batch size $B$.} 
\label{alg:layerwise-sgd}
\begin{algorithmic}[1]
\Require Training data $\{(x_t, y_t)\}_{t \ge 1}$, model: $\NN(x;(w,a,b)) = \sum_{i=1}^n a_i \sigma(w_i^\top x +b_i)$.
\State {\bfseries Initialize}: first-layer weights $w^{(0)} \sim \mathcal{N}(0, I_d/ d)$ (normalized variance), $a^{(0)} = \tau$, $b^{(0)} =0$.
\vspace{0.5em}
\Statex \textbf{Phase 1: Train the first layer (second layer frozen)}
\For{$t=0$ {\bfseries to} $T_1-1$ {\bfseries :}} 
    \State Sample a fresh mini-batch $S^{(t)}=\{(x_s, y_s)\}_{s=1}^B$
    \State Update 
    \(
    w^{(t+1)} \leftarrow w^{(t)} - \gamma_1 \nabla_{w^{(t)}} L(S^{(t)};w^{(t)},a^{(0)},b^{(0)})
    \)
\EndFor
\vspace{0.5em}
\Statex \textbf{Phase 2: Train the second layer (first layer frozen)}
\State Draw random biases: $\hat b_1,\dots,\hat b_n \sim \Unif[-A,A]^{\otimes n}$
\For{$t=0$ {\bfseries to} $T_2-1$ {\bfseries :}}
    \State Sample a fresh mini-batch $S^{(t)}=\{(x_s, y_s)\}_{s=1}^B$
    \State Update $a^{(t+1)} \leftarrow a^{(t)} - \gamma_2 \nabla_{a^{(t)}} L(S^{(t)};w^{(T_1)},a^{(t)},\hat b)$
\EndFor
\vspace{0.5em}
\State \textbf{Output:} Trained model $\NN(x;(w^{(T_1)},a^{(T_2)},\hat b))$
\end{algorithmic}
\end{algorithm}
\noindent 

\paragraph{Layerwise-SGD with square loss.} We consider learning with a 2-layer neural network:
\[
\NN(x;\theta^t) = \sum_{h=1}^n a_{h} \sigma\bigl((w_{h}^t)^\top x +b_h\bigr),
\]
with $n$ hidden units, and polynomial activation $\sigma$, satisfying Assumption~\ref{ass:activation}. We initialize $w_{hj} \sim \cN(0, 1/d)$ and $a_h =\tau$, for all $h \in [n], j \in [d]$.  We perform layerwise training: i.e., during a first phase we train the first layer's weights for one step, keeping the second layer's weights fixed, and during a second phase we train the second layer's weights, keeping the $w_h$ fixed. The bias neurons are initialized at $0$ and kept fixed during the first step. Before the second phase, we draw $b_h \overset{iid}{\sim} \Unif[-A,A]$, for a fixed $A>0$, and keep them frozen during training. We adopt this choice by convenience of the analysis. We note, however, that with further technical work, one could consider random biases for the first part of training as well. We use the squared loss: $\frac 12 ( y - \hat y)^2$. 
We refer to Algorithm~\ref{alg:layerwise-sgd} for a detailed description of the layerwise SGD algorithm that we use.
For each hidden unit $h \in [n]$,
\begin{align}
    w_h^{t+1} & =  w_h^t - \frac{\gamma}{2} \frac{1}{B}\sum_{b =1}^B \nabla_{w_h^t} \big(y^b - \NN(x^b;\theta^t)\big)^2
\end{align}
where $B \in \bN$ denotes the batch size. 
Note that $|\NN(x;\theta^t)| = \tau\,\mathrm{polylog}(d)$ with high probability
over the initialization and the mini-batch, for $B = \mathrm{poly}(d)$. Indeed,
$x$ is centered (by symmetry of $s$) and sub-Gaussian (since $|s_i|=1$ and
$\xi_j$ is sub-Gaussian), so $(w_{h}^0)^\top x$ is sub-Gaussian with variance
$O(1)$, hence $|\sigma((w_{h}^0)^\top x)| = O(\mathrm{polylog}(d))$ and
$|a_{h}| = O(\tau)$; a union bound over $h \in [n]$ and $b \in [B]$ gives the
polylog factor (using $n, P = O_d(1)$).
Thus, for $\tau$ small enough, we have
\begin{align}
   - \gamma \nabla_{w_h^t} \frac{1}{2} (y - \NN(x;\theta^t))^2
   & =  \gamma a_h y \sigma'(w_h^t x) x 
       - \gamma \NN(x;\theta^t) a_h \sigma'(w_h^t x) x \\
   & =  \gamma \tau y \sigma'(w_h^t x) x 
       - \gamma O\!\left(n \tau^2 \poly\log(d)\right),
\end{align}
with high probability.
Thus, for $\tau$ small, the interaction term
$\NN(x;\theta^t)a_h\sigma'(w_h^t x)x$ is negligible, and one can track only the correlation term. A similar approach is used in e.g.~\cite{abbe2023sgd,mousavi2023gradient}.

\subsection{First layer's training}
\label{sec:proof_first_layer}

\paragraph{Initial population gradient.} In the following, let us fix a hidden neuron $h \in [n]$, and for simplicity let us denote $w:=w_h^0$. Recall, we initialize $a_h=\tau$ and $b_h =0$ for all $h \in [n]$. For $i \in [N]$, let us denote:
\begin{align}
    (\Delta_\xi(w))_i := \sum_{j \in C_i} \xi_{j}w_{j}.
\end{align}
For $j \in C_i$, the correlation term of the initial population gradient reads:
\begin{align} \label{eq:pop_gradient}
    \bar G(w_{j}) :& = \frac 1 \tau \E_{x,y} [ y \cdot \partial_{w_j}\NN(x;\theta^0) ] \\
    & = \E_s [f(s) s_i \E_{\xi}[ \xi_{j} \sigma'(s^\top \Delta_{\xi}(w))] ].
\end{align}

In the following, we show that $\overline G(w_j)$ concentrates around a quantity that depends only on the cluster's latent $i$.

\begin{lemma}[Initial Population Gradient]\label{lem:pop_grad}
Assume the conditions of Theorem~\ref{thm:clusters}; in particular, the noise
variables $\{\xi_j\}_{j \in [d]}$ are mutually independent, sub-Gaussian with
$\|\xi_j\|_{\psi_2} \le \psi$, with means $m_j := \E[\xi_j]$, variances
$\sigma_j^2 := \Var(\xi_j)$, and bounded third moments
$\sup_j \E|\xi_j - m_j|^3 = O_d(1)$.
Fix a neuron $w \in \bR^d$ and define
\[
(\Delta_\xi(w))_i := \sum_{j\in C_i} \xi_j w_j,
\qquad
(\Delta_m(w))_i := \sum_{j\in C_i} m_j w_j,
\qquad
V^2 := \sum_{j=1}^d w_j^2\,\sigma_j^2.
\]
Then for every $i \in [N]$ and every $j \in C_i$,
\[
\overline G(w_j) \;=\; m_j\,\alpha_i(w) \;+\; \eta_j,
\]
where
\begin{equation}\label{eq:alpha}
\alpha_i(w) \;:=\; \E_s\!\Big[f(s)\,s_i\,
\E_G\,\sigma'\!\big(s^\top \Delta_m(w) + G\big)\Big],
\qquad G \sim \cN(0, V^2),
\end{equation}
and, on with high probability,
\[
|\eta_j| \;\le\; C_1\,|w_j| \;+\; C_2\,w_j^2 \;+\; C_3 \sqrt{\frac{\log d}{d}}.
\]
The constants $C_1, C_2, C_3 > 0$ depend only on the activation $\sigma$ and on
the sub-Gaussian parameters and third moments of $\xi$.
\end{lemma}


\noindent
The proof of lemma~\ref{lem:pop_grad} is deferred to Appendix~\ref{app:proof_pop_grad}.

\paragraph{Initial estimated gradient.} 
We now consider the gradient estimated through $B$ i.i.d.\ samples:
\begin{align}
    G_B(w) := \frac{1}{B} \sum_{b=1}^B y^b \, \sigma'(w x^b) \, x^b,
\end{align}
where we recall that $\sigma$ is a polynomial activation of degree $P$ such that $2^N \leq P = O_d(1)$. 
If $B$ is large enough, the estimated gradient is close to the population gradient $\bar G(w)$, as formalized by the following Lemma.

\begin{lemma}\label{lem:estimated_G_poly}
Let $M_{f,N} = \max_{s \in \{\pm 1\}^N} |f(s)|$. 
There exists a constant $K_N > 0$ (depending only on the polynomial coefficients of $\sigma$ 
and on the subgaussian parameter of $x$) such that, if
\begin{align}
    B \ge C \, \frac{K_N^2 M_{f,N}^2}{\zeta^2} \, \log(dn)
\end{align}
for some universal constant $C > 0$, then with probability at least $1 - d^{-C^*}$, for some $C^* > 0$, 
we have, for all $h \in [n]$,
\begin{align}
    \| G_B(w_h^0) - \bar G(w_h^0) \|_{\infty} \le \zeta.
\end{align}
\end{lemma}

\begin{proof}
Fix $h \in [n]$ and a coordinate $j \in [d]$. 
Define
\begin{align}
    Z_b := y^b \, \sigma'(w_h^0 x^b) \, x_j^b.
\end{align}
Since $|y^b| \le M_{f,N}$ and $x$ is subgaussian while $w_h^0 \sim \mathcal{N}(0, I_d/d)$ 
is independent of $x$, the inner product $w_h^0 x$ is subgaussian. 
A polynomial function of a subgaussian variable 
multiplied by a subgaussian coordinate is subexponential. 
Hence, there exists a constant $K_N > 0$ such that 
$\| Z_b \|_{\psi_1} \le K_N M_{f,N}$.

Applying Bernstein's inequality for subexponential random variables yields, 
for some constants $c, C > 0$,
\begin{align}
    \Pr\!\left(
        \Big| \frac{1}{B} \sum_{b=1}^B Z_b - \mathbb{E}[Z_b] \Big| > \zeta
    \right)
    \le 2 \exp\!\left(
        -c B \min\!\left\{
            \frac{\zeta^2}{K_N^2 M_{f,N}^2},
            \frac{\zeta}{K_N M_{f,N}}
        \right\}
    \right).
\end{align}
Choosing 
$B \ge C \, \frac{K_N^2 M_{f,N}^2}{\zeta^2} \, \log(dn)$ 
makes the right-hand side at most $(dn)^{-C^*}$ for some $C^* > 0$. 
A union bound over all $j \in [d]$ and $h \in [n]$ 
then gives the desired result.
\end{proof}

\noindent
By the above, it follows that:
\begin{align} \label{eq:first_layer}
    w_j^1 = w^0 + \gamma \tau  (m_j \alpha_i + \eta_j + \omega_j + \zeta),
\end{align}
where $\alpha_i$ is defined in~\eqref{eq:alpha}, $ |\eta_j | \leq C_1|w_j| + C_2w_j^2 + C_3 \sqrt{\log(d)/d}$ (Lemma~\ref{lem:pop_grad}), $\omega_j = O(n\tau \poly \log(d))$ (control of the interaction component of the gradient), and $\zeta$ follows by Lemma~\ref{lem:estimated_G_poly}, assuming the batch size $B =\Omega(\log(dn)/\zeta^2)$.

\subsection{Second layer's training}
We now show that training the second layer, while keeping the first layer fixed to~\eqref{eq:first_layer}, allows to achieve small generalization error. the proof follows by the following steps:
\begin{enumerate}
    \item We show that the $(\alpha_i)_{i \in [N]}$ are such that any $f: \{ \pm 1 \}^N \to \bR$ is projection-consistent. Specifically, we show that for all $s,t \in \{ \pm 1 \}^N$ such that $f(s) \neq f(t)$, and for all $\epsilon>0$, we have 
    \begin{align}
        |\alpha^T (s-t)| > \epsilon,
    \end{align}
    with probability at least $1-\poly(\epsilon)$ over the initialization of the first layer's weights (Lemma~\ref{lem:non_degeneracy}). This makes use of the well-known Carbery-Wright inequality~\cite{carbery2001distributional}.
    \item Next, we show that the property above, and randomly drawn bias neurons, guarantee that there exists an assignment of the second layer's weights such that 
    \begin{align}
        f(s) \approx \sum_{h=1}^n a_h^* \sigma(\gamma v_i s^T \alpha +b_h)  
    \end{align}
    for all $s \in \{\pm 1 \}^N$ (Lemma~\ref{lem:certificate}). 
    \item Then, we show that the generalization error of the certificate of the previous point is low (Lemma~\ref{lem:generalization}). This requires bounding the errors appearing in~\eqref{eq:first_layer}.
    \item Lastly, we use standard results on the convergence of SGD on convex losses to conclude that the second layer's training will convergence to small generalization error (Theorem~\ref{thm:convergence_SGD_martingale}).
\end{enumerate}

Let us start with the first step. 

\begin{lemma}[Non-degeneracy] \label{lem:non_degeneracy}
Let $Z \sim \cN(0,{\rm diag}(v_i)_{i \in [N]}/d)$, and let $v_{\rm sum}: =  \sum_{i \in [N]} v_i$. Assume $v_i \geq v_{\min}>0$, for all $i \in [N]$. Let $\sigma$ be a polynomial of degree $P$ that satisfies assumption~\ref{ass:activation} and let $\sigma'$ be its first derivative. Let $f: \{ \pm 1 \}^N \to \bR$ be a target function and assume $\Var(f)> \zeta>0$. For $c \in \{0, \pm 2 \}^N$, define:
\begin{align}
    \cP_Z(c):=\E_s \Big[ f(s)\cdot \Big( \sum_{j=1}^N \tilde C_j v_j c_j s_j \Big) \cdot \E_G \sigma'(s^T Z+G)\Big],
\end{align}
where $G \sim \cN(0,V^2)$, for some $V>0$ and where $\tilde C_j>0$ for $j \in [N]$.
 Let $\cL_f := \{ (t,y)\in \{ \pm 1\}^N \times \{ \pm 1 \}^N : f(t) \neq f(y) \}$.
Then, there exist constants $C_1,C_2>0$ and $k_0 \in [P]$, such that for all $\epsilon>0$, 
\begin{align}
    \pr_Z \Big( \min_{(t,y) \in \cL_f} \Big| \cP_Z(t-y) \Big| > \epsilon \cdot  v_{\min} \left(\frac{v_{\min}}{d}\right)^{k_0/2} \Big) > 1-C_1P | \cL_f | \left(\frac{\epsilon}{C_2}\right)^{1/P}.
\end{align}
\end{lemma}

\begin{proof} 
    Write $\sigma'(x):=\sum_{n=1}^P a_n x^n$.
    Define $\psi(t):=\E_{G\sim \cN(0,V^2)}[\sigma'(t+G)]$. This is a polynomial of degree $\leq P$, and for fixed $c$, $\cP_Z(c) = \E_s[f(s)h_c(s) \psi(s^TZ)]$, $h_c(s) = \sum_{j=1}^N \tilde C_j v_j c_j s_j$. Note that also $\cP_Z(c)$ is a polynomial in $Z$ of degree at most $P$.
    Let us state the following lemma, which is a restatement of the Carbery-Wright inequality~\cite{carbery2001distributional}. 
    \begin{lemma}[\cite{carbery2001distributional}] \label{lem:CW}
        Let $\phi: \bR^N \to \bR$ be a non-zero polynomial of degree at most $P=O(1)$, and let $\mu$ be a log-concave probability measure over $\bR^N$. Then, for all $\alpha>0$,
        \begin{align}
            \mu \left( \{ x: | \phi(x)| \leq \alpha \| \phi\|_{L^2(\mu)}\} \right) \leq C P \alpha^{1/P},
        \end{align}
        where $C>0$ is an absolute constant.
    \end{lemma}
    \noindent 
    By applying Lemma~\ref{lem:CW} to our setting with $\mu$ being the Gaussian law of $Z$, and by union bound, we get that there exists a constant $C$ such that with probability at least $1- CP|\cL_f|\epsilon^{1/P}$ over $Z$,
    \begin{align}
        \min_{(t,y) \in \cL_f } | \cP_Z(t-y) | > \epsilon \cdot \sqrt{M_{\min}}, 
    \end{align}
where $ M_{\min} := \min_{(t,y) \in \cL_f } \E_{Z} [\cP_{Z}(t-y)^2] $. 

Thus everything reduces to a lower bound on $M_{\min} $. Fix a pair $(t,y)$ and let $c:=t-y$. Recall that $h_c(s):=\sum_{j=1}^N \tilde C_j v_j c_j s_j $. Then,
\begin{align}
    \E_{Z} [\cP_{Z}(c)^2] & = \E_Z \left[ \E_s[f(s)  h_c(s)  \psi(s^TZ) ]^2 \right]\\
    & = \E_{s,s'}\left[ f(s)f(s') h_c(s) h_c(s') \E_Z[\psi(s^TZ)\psi((s')^TZ)] \right].
\end{align}
Let $\zeta^2 := \Var (s^TZ) = v_{\rm sum}/d$ be the variance of $s^TZ$. Let $x \sim \cN(0,1)$. As usual, we can write the Hermite expansion,
\begin{align}
    \psi(\zeta x) = \sum_{k=1}^P \beta_k(\zeta) H_k(x), \qquad \beta_k(\zeta) = \frac{1}{k!} \E_x[\psi(\zeta x) H_k(x)].
\end{align}
Note that:
    \begin{align}
        \beta_k(\zeta):&=  \frac{\zeta^k}{k!} S^{(k)}_{V^2+\zeta^2}(0), 
    \end{align}
    where $S_{V^2+\zeta^2}(t) :=\E_{G \sim \cN(0,V^2+\zeta^2)}[\sigma'(t+G)]$. By Assumption~\ref{ass:activation}, for each $1 \leq k \leq N$, $S^{(k)}_{V^2+\zeta^2}(0)\geq c_0>0$. For two vectors $s,s' \in \{ \pm 1 \}^N$, write the correlation:
\begin{align}
    \E_Z[\psi(s^TZ)\psi((s')^TZ)]& = \sum_{k=0}^P \beta_k(\zeta)^2 k! \rho(s,s')^k,
\end{align}
where $\rho(s,s'):=\frac{1}{v_{\rm sum}} \sum_{i=1}^N s_i s'_i v_i $.
Plugging it in the above, and using a tensorization trick, we can thus write:
\begin{align}
    \E_{Z} [\cP_{Z}(c)^2] & = \sum_{k=0}^P \beta_k(\zeta)^2 k! \frac{1}{v_{\rm sum}^k} \sum_{i_1,...i_k} v_{i_1}...v_{i_k}\cdot \E_s[ f(s) h_c(s) s_{i_1}...s_{i_k} ]^2\\
    & \geq \sum_{k=0}^P \beta_k(\zeta)^2 k! \left(\frac{v_{\min} }{v_{\rm sum}}\right)^k \max_{T \subseteq[N]: |T|\leq k} \E_s[f(s)h_c(s) \chi_T(s)]^2
\end{align}
By Assumption~\ref{ass:activation}, there exists some fixed $k_0\in [P]$ and $c_1>0$ such that 
\begin{align}
    \E_{Z} [\cP_{Z}(c)^2] \geq c_1 \left(\frac{v_{\rm sum}}{d}\right)^{k_0} \left( \frac{v_{\min}}{v_{\rm sum}} \right)^{k_0}\max_{T \subseteq[N]: |T|\leq k_0} \E_s[f(s)h_c(s) \chi_T(s)]^2.
\end{align}
The following claim bounds the right-most term.
\begin{claim}\label{cl:fourier-lower-bound}

There exists a subset $T\subseteq[N]$ and a constant $C>0$ such that 
\[
\mathbb{E}_s\big[f(s)h_c(s)\chi_T(s)\big]^2
\;\ge\;
2^{2-2N}\, C \,v_{\min}^2.
\]
\end{claim}

\begin{proof}[Proof of Claim~\ref{cl:fourier-lower-bound}]
Fix $(t,y)$ with $f(t)\neq f(y)$. Let $D := \{i\in[N] : t_i \neq y_i\}$.
Since $t\neq y$, we have $D\neq\emptyset$.
For $i\in D$, we have $t_i-y_i \in \{\pm2\}$ and for $i\notin D$ the coefficient is zero. Compute $h_c$ at $t$:
\begin{align*}
h_c(t)
&= \sum_{i=1}^N \tilde C_i v_i (t_i - y_i) t_i
 = \sum_{i\in D} \tilde C_i v_i 2 t_i^2 = 2 \sum_{i\in D} \tilde C_i v_i.
\end{align*}
Using $v_i \ge v_{\min}$ and $\tilde C_i \ge \tilde C_{\min}$, and $D\neq\emptyset$, we obtain
$|h_c(t)|
= 2\Big|\sum_{i\in D} \tilde C_i v_i\Big|
\;\ge\; 2\,\tilde C_{\min}\,v_{\min}$. Similarly, $h_c(y) = -h_c(t)$, so $|h_c(y)| = |h_c(t)| \ge 2\tilde C_{\min}v_{\min}$. Since $f(t)\neq f(y)$, at least one of $f(t), f(y)$ is nonzero.
Define
\[
s_0 \in \{t,y\} \quad\text{such that}\quad |f(s_0)| = \max\{|f(t)|,|f(y)|\}\geq c_0>0.
\]
Thus, denoting $u(s)=f(s) h_c(s)$, we have
$\max_{s\in\{\pm1\}^N} |u(s)|
\;\ge\;
2\tilde C_{\min} v_{\min} c_0.$
Now consider the Fourier-Walsh expansion of $u$: $u(s) = \sum_{T\subseteq[N]} \hat u(T)\,\chi_T(s),$ with $\hat u(T) = \mathbb{E}_s[u(s)\chi_T(s)].$ For any fixed $s$,
$|u(s)|
= \Big|\sum_{T} \hat u(T)\chi_T(s)\Big|
\le \sum_{T} |\hat u(T)|.$
Taking the maximum over $s$ gives
\[
\sum_{T} |\hat u(T)|
\;\ge\;
\max_s |u(s)|
\;\ge\;
2\tilde C_{\min} v_{\min} c_0.
\]
On the other hand,
\[
\sum_{T} |\hat u(T)|
\le 2^N \max_{T} |\hat u(T)|.
\]
Combining these two inequalities yields
\[
\max_{T\subseteq[N]} |\hat u(T)|
\;\ge\;
\frac{2\tilde C_{\min} v_{\min} c_0}{2^N}
= 2^{1-N}\tilde C_{\min} v_{\min} c_0.
\]
Hence there exists some $T\subseteq[N]$ (depending on $t,y$) with
\[
|\hat u(T)|
= \Big|\mathbb{E}_s[f(s)h_c(s)\chi_T(s)]\Big|
\;\ge\;
2^{1-N}\tilde C_{\min} v_{\min} c_0,
\]
and therefore
\[
\mathbb{E}_s[f(s)h_c(s)\chi_T(s)]^2
= \hat u(T)^2
\;\ge\;
2^{2-2N}\,\tilde C_{\min}^2\,v_{\min}^2\,c_0.
\]
This proves the claim.
\end{proof}
The lemma follows by direct application of the claim.

\end{proof}

Now, we proceed to building a certificate for our target function $f$, as linear combination of the representation learned by the first layer.

In order to do that, let $(\tilde \alpha_i)_{i \in [N]}$ be such that  $\tilde \alpha_i := \gamma \tau \alpha_i v_i$, and define the projection values
\(
u(s):=s^\top \tilde\alpha,\) \(s\in\{\pm1\}^N,
\)
and let
\(
U:=\{u_1<\cdots<u_M\}
\)
be the set of distinct values taken by $u(s)$, so that $M\le 2^N$. Define also the minimum gap
\[
\Delta:=\min_{m\neq m'} |u_m-u_{m'}|.
\]
On the event of Lemma~\ref{lem:non_degeneracy}, $f$ is projection-consistent, i.e.
\(
u(s)=u(t)\;\Longrightarrow\; f(s)=f(t).
\)

\begin{lemma}[Certificate]
\label{lem:certificate}
Assume that Assumptions~\ref{ass:data} and~\ref{ass:activation} are satisfied.
Let $\alpha \in \mathbb{R}^N$ be defined by~\eqref{eq:alpha}, and define
$(\tilde \alpha_i)_{i \in [N]}$ by $\tilde \alpha_i := \gamma \tau \alpha_i v_i$.
Let $b_h \overset{iid}{\sim} \Unif[-A,A]$ for $h \in [n]$, with \(
A\in[\Delta/C_A,\; C_A\Delta]
\),
for some fixed constant $C_A\ge 1$, and $ A \geq 
\max_m |u_m|$.
Fix $\epsilon \in (0,1)$.
Then there exist constants $C_0,C_1,C_2,C_{\sigma,P,N,C_A}>0$ such that if
\[
n \;\ge\; C_0\, 2^N \big(\log 2^N + \log(1/\epsilon)\big),
\]
there exists $a^\star \in \mathbb{R}^n$ with: 
\begin{align}
    \|a^\star\|_2^2
\;\le\;
C_{\sigma,P,N,C_A}\,
\frac{M_{f,N}^2}{n\,\Delta^{2P}} .
\label{eq:certificate-a2}
\end{align}
such that on the event of Lemma~\ref{lem:non_degeneracy}, for all $s \in \{ \pm 1 \}^N$:
\begin{align}
f(s) = \sum_{h=1}^n a_h^\star \sigma(s^\top \tilde \alpha + b_h).
\label{eq:certificate-approx}
\end{align}
\end{lemma}

\begin{proof}
Let
\(
M_{f,N}:=\sup_{s\in\{\pm1\}^N}|f(s)|.
\)
Since $|u_m|\le \|\tilde\alpha\|_1$ and
$|b_h|\le A$, and since $\sigma$ is a degree-$P$ polynomial, there exists a constant
$C_\sigma>0$ such that
\[
|\sigma(u_m+b_h)| \le C_\sigma A^P
\qquad \forall\, m\in[M],\ \forall\, h\in[n].
\]

Define the feature matrix $\Phi\in\mathbb R^{M\times n}$ by
\(
\Phi_{m,h}:=\sigma(u_m+b_h),
\)
and its normalized version
\(
\Psi:=A^{-P}\Phi,
\)
so that $\Phi=A^P\Psi$ and
\(
|\Psi_{m,h}|\le C_\sigma.
\)

Now let
\[
K:=\E_b[\psi(b)\psi(b)^\top]\in\mathbb R^{M\times M},
\qquad
\psi(b):=\big(\sigma(u_1+b)/A^P,\dots,\sigma(u_M+b)/A^P\big),
\]
where $b\sim\Unif[-A,A]$.
Since $\sigma$ has degree $P$, nonzero leading coefficient, and $M\le 2^N\le P+1$,
the shifted polynomials
\(
b\mapsto \sigma(u_m+b),\) \( m=1,\dots,M,
\)
are linearly independent whenever the shifts $u_1,\dots,u_M$ are distinct. Therefore,
for every $c\in\mathbb R^M\setminus\{0\}$,
\[
c^\top K c
=
\E_b\!\left[\left(\sum_{m=1}^M c_m\,\frac{\sigma(u_m+b)}{A^P}\right)^2\right]
>0,
\]
and hence $K\succ 0$. We may therefore define
\(
\kappa:=\lambda_{\min}(K)>0.
\)

\begin{claim}\label{claim:kappa}
For every fixed constant $C_A\ge 1$, there exists a constant
$c_{\sigma,P,N,C_A}>0$ such that the following holds. If
\(
A\in[\Delta/C_A,\; C_A\Delta],
\)
then
\[
\kappa=\lambda_{\min}(K)\ge c_{\sigma,P,N,C_A}.
\]
\end{claim}
Let us first show how Claim~\ref{claim:kappa} implies the Lemma. 
We work on the event of Lemma~\ref{lem:non_degeneracy}, on which $f$ is projection-consistent.

Define the empirical Gram matrix
\[
\widehat K := \frac{1}{n}\Psi\Psi^\top
= \frac{1}{n}\sum_{h=1}^n \psi(b_h)\psi(b_h)^\top .
\]
Set $X_h := \psi(b_h)\psi(b_h)^\top \succeq 0$, so that $\widehat K = \frac{1}{n}\sum_{h=1}^n X_h$
and $\mathbb{E}[X_h]=K$.
Moreover, since $|\psi_m(b)|\le C_\sigma$ for all $m\in[M]$, we have
\[
\|X_h\|_2=\|\psi(b_h)\|_2^2 \le M C_\sigma^2.
\]
By a matrix Chernoff bound (e.g.\ \cite{tropp2012user}),
there exist a numerical constant $c>0$ such that for any $\delta\in(0,1)$,
\[
\Pr\!\left(\lambda_{\min}(\widehat K) \le (1-\delta)\kappa\right)
\le M\exp\!\left(-c\,\frac{\delta^2\,n\,\kappa}{M C_\sigma^2}\right),
\]
where $\kappa:=\lambda_{\min}(K)>0$. 
Choosing $\delta=\tfrac12$ and taking
\[
n \;\ge\; C\,\frac{M C_\sigma^2}{\kappa}\big(\log M+\log(1/\epsilon)\big),
\]
yields, with probability at least $1-C_1\epsilon^{C_2}$,
\[
\lambda_{\min}(\widehat K)\ge \frac{\kappa}{2}.
\]

Equivalently,
\begin{equation}
\lambda_{\min}(\Phi\Phi^\top)
= A^{2P}\lambda_{\min}(\Psi\Psi^\top)
\ge \frac{\kappa n}{2}\, A^{2P}.
\label{eq:gram-lower}
\end{equation}

By projection-consistency, define $F\in\mathbb{R}^M$ by $F_m= f(s_m)$, where $s_m$ is such that $u_m = u(s_m)$, and note that
$\|F\|_2 \le \sqrt{M}\,M_{f,N}$.
Let $a^\star$ be the interpolating solution:
\[
a^\star := \Phi^T (\Phi \Phi^T)^{-1} F.
\]
Then $\Phi a^\star = F$, thus
\[
\sup_{s\in\{\pm1\}^N}\Big|f(s)-\sum_{h=1}^n a_h^\star\sigma(s^\top\tilde\alpha+b_h)\Big|=0.
\]

Now, let us upper bound the norm of $a^\star$.
Using $\|\Phi\|_2 \le \|\Phi\|_F$ and the bound $|\sigma(u_m+b_h)|\le C_\sigma A^P$,
we have
\[
\|\Phi\|_2 \le C_\sigma \sqrt{Mn}\,A^P.
\]
Moreover, $\|F\|_2 \le \sqrt{M}\,M_{f,N}$. Therefore,
\[
\|a^\star\|_2
\le \|\Phi\|_2\,\|(\Phi\Phi^\top)^{-1}\|_2\,\|F\|_2
\le \frac{\|\Phi\|_2}{\lambda_{\min}(\Phi\Phi^\top)}\,\|F\|_2.
\]
Combining with~\eqref{eq:gram-lower}, namely
$\lambda_{\min}(\Phi\Phi^\top)\ge \frac{\kappa n}{2}A^{2P}$, yields
\[
\|a^\star\|_2^2
\le C_{\sigma,P,N,C_A}\,\frac{M^2\,M_{f,N}^2}{n\,A^{2P}} .
\]
Since $M\le 2^N$ and $N=O_d(1)$, we may absorb the factor $M^2$ into the constant, and therefore
\[
\|a^\star\|_2^2
\le C_{\sigma,P,N,C_A}\,\frac{M_{f,N}^2}{n\,A^{2P}} .
\]
Since \(A\in[\Delta/C_A,\; C_A\Delta]\), it follows that
\[
\|a^\star\|_2^2
\le C_{\sigma,P,N,C_A}\,\frac{M_{f,N}^2}{n\,\Delta^{2P}} .
\]

All is left is the proof of Claim~\ref{claim:kappa}.
\end{proof}

\begin{proof}[Proof of Claim~\ref{claim:kappa}]
Let us write
\[
\sigma(x)=a_P x^P+a_{P-1}x^{P-1}+\cdots+a_0
\]
with $a_P\neq 0$. Let $t_m := u_m/A$ and $z:=b/A \sim \Unif[-1,1]$.
Since
\[
A\in[\Delta/C_A,\; C_A\Delta],
\qquad
\Delta=\min_{m\neq m'}|u_m-u_{m'}|,
\]
we have
\[
|t_m-t_{m'}|=\frac{|u_m-u_{m'}|}{A}\ge \frac{1}{C_A},
\qquad \forall\,m\neq m'.
\]
Define
\[
\phi_A(z)\ :=\ \bigg(\frac{\sigma(A(t_1+z))}{A^P},\dots,\frac{\sigma(A(t_M+z))}{A^P}\bigg),
\qquad
K(A,t):=\E[\phi_A(z)\phi_A(z)^\top].
\]
Then, for each $m$,
\[
\frac{\sigma\!\big(A(t_m+z)\big)}{A^P}
=
a_P(t_m+z)^P
+
\sum_{k=0}^{P-1} a_k A^{k-P}(t_m+z)^k.
\]
Since $|t_m|\le 1$ and $|z|\le 1$, we have $|t_m+z|\le 2$, and therefore the remainder
term is uniformly $O(1/A)$ over $z\in[-1,1]$ and over all choices of
$(t_1,\dots,t_M)\in[-1,1]^M$ satisfying
\[
t_1<\cdots<t_M,
\qquad
t_{m+1}-t_m\ge \frac{1}{C_A}.
\]
Consequently,
\[
\phi_A(z)
\;\longrightarrow\;
\phi_\infty(z)
:=
\big(a_P(t_1+z)^P,\dots,a_P(t_M+z)^P\big)
\]
uniformly on $[-1,1]$ as $A\to\infty$, uniformly over all such separated tuples
$(t_1,\dots,t_M)$.
It follows that $K(A,t)$ converges in operator norm to
\[
K_\infty:=\E\big[\phi_\infty(z)\phi_\infty(z)^\top\big]
\]
uniformly over the same class of tuples.

Let $c\in\mathbb{R}^M\setminus\{0\}$ and define
\[
g_{\infty,c}(z):=\sum_{m=1}^M c_m\,a_P(t_m+z)^P.
\]
Since $M\le 2^N\le P+1$ and the shifts $t_1,\dots,t_M$ are distinct,
the functions $z\mapsto (t_m+z)^P$ are linearly independent.
Hence $g_{\infty,c}$ is not identically zero, and therefore
\[
c^\top K_\infty c
=\E\big[g_{\infty,c}(z)^2\big]
>0.
\]
Thus $K_\infty\succ 0$, and \(
\kappa_\infty:=\lambda_{\min}(K_\infty)>0.
\)

Since the set of separated tuples
\[
\mathcal T_{C_A}
:=
\Big\{(t_1,\dots,t_M)\in[-1,1]^M:\ t_1<\cdots<t_M,\ t_{m+1}-t_m\ge 1/C_A\Big\}
\]
is compact, and $t\mapsto \lambda_{\min}(K_\infty(t))$ is continuous and strictly positive
on $\mathcal T_{C_A}$, there exists
\[
\kappa_{\infty,C_A}>0
\]
such that
\[
\lambda_{\min}(K_\infty)\ge \kappa_{\infty,C_A}
\]
uniformly over $\mathcal T_{C_A}$.
Hence, by uniform convergence and Weyl's inequality, there exists $A_1\ge 1$ such that
for all $A\ge A_1$,
\[
\lambda_{\min}(K(A,t))\ge \frac{\kappa_{\infty,C_A}}{2}.
\]

Finally, for $A\in[1,A_1]$, the map $(A,t)\mapsto \lambda_{\min}(K(A,t))$ is continuous
on the compact set $[1,A_1]\times \mathcal T_{C_A}$ and is strictly positive there.
Therefore it attains a positive minimum on that set. Combining the two regimes yields
a constant $c_{\sigma,P,N,C_A}>0$ such that
\[
\kappa=\lambda_{\min}(K)\ge c_{\sigma,P,N,C_A},
\]
which proves the claim.
\end{proof}


\begin{lemma}[Generalization error] \label{lem:generalization}
Let $\hat \theta :=(w^1,a^*, b)$. Then, for any $\delta>0$,
under the assumptions of Lemma~\ref{lem:certificate}, 
\begin{align}
    \E_{x,y}[|y - \NN(x;\hat \theta)|] < \delta,
\end{align}
for sufficiently large $d$.
\end{lemma}

\begin{proof}
Recall that, for each $j \in C_i$, and for each $h \in [n]$:
\begin{align}
    w^1_{j,h} =w^0_{j,h} +\gamma \tau m_{j} \alpha_i + \gamma \tau \eta_{j,h}' ,
\end{align}
with
\begin{align}
    \eta_{j,h}' =  \zeta + \omega_{j,h} + \eta_{j,h},
\end{align}
where $\alpha_i$ is defined in~\eqref{eq:alpha}, $ |\eta_j | \leq C_1|w_j| + C_2w_j^2 + C_3 \sqrt{\log(d)/d}$ (Lemma~\ref{lem:pop_grad}), $\omega_j = O(n\tau \poly\log(d))$, and $\zeta$ follows by Lemma~\ref{lem:estimated_G_poly}, assuming the batch size $B =\Omega(\log(dn)/\zeta^2)$. 
Note that by Lemma~\ref{lem:preliminaries}, $|w_{j,h}^0|=O(\log(dn)/\sqrt{d})$ for all $j,h$.
Thus, for a fixed $h$,
\begin{align}
    (w^1_h)^Tx & = \gamma \tau \sum_{i=1}^N \alpha_i s_i \sum_{j \in C_i} m_j \xi_j + \sum_{j=1}^d (\gamma \tau \eta_j' + w_{j}^0) x_j \\
    & =\gamma \tau \sum_{i=1}^N \alpha_i s_i v_i + \gamma \tau  \sum_{i=1}^N \alpha_i s_i   \sum_{j \in C_i} m_j (\xi_j -m_j) + \sum_{j=1}^d (\gamma \tau  \eta_j' + w_{j}^0) x_j
\end{align}
Let $\NN(s;\theta^*) := \sum_{h=1}^n a_h^* \sigma(\tilde \alpha^T s + b_h) $, with $\tilde \alpha_i = \gamma \tau \alpha_i v_i$,
\begin{align}
   \E_{x,y}[|y - \NN(x;\hat \theta)|] & \leq   \E_{x,y}[|y- \NN(s;\theta^*)|] + \E_{x,y}[|\NN(s;\theta^*)- \NN(x;\hat \theta)|] \\
   & \overset{(a)}{\leq} \sum_{i=1}^n |a_i^*| \cdot  \E_{x,y} \left[ |\sigma(\tilde \alpha^Ts+b)  -\sigma((w^1)^T x+b) | \right]
\end{align}
where $(a)$ holds because, by Lemma~\ref{lem:certificate}, the first term is zero. By mean value theorem for polynomials,
\begin{align}
    |\sigma(\tilde \alpha^Ts+b)  -\sigma((w^1)^T x+b) |  &\leq | \sigma'(t)|\cdot  | \tilde \alpha^Ts- (w^1)^Tx|\\
    & \leq C_\sigma (1+|t|^{P-1}) \cdot  | \tilde \alpha^T s- (w^1)^Tx|
\end{align}
with $ t $ between $\tilde \alpha^Ts$ and $ (w^1)^Tx$ and $C_\sigma:=\max_{i \in [P]} P |\beta_k|$. By Cauchy-Schwartz and by sub-Gaussianity of the $x$, we get
\begin{align}
    \E_{x}[ |\sigma(\tilde \alpha^Ts +b)  -\sigma((w^1)^T x+b) | ] \leq C_\sigma \sqrt{\E(1+|t|^{P-1})^2}\cdot \sqrt{\E_{x,y}[( (w^1)^Tx - \tilde \alpha^Ts)^2]}
\end{align}
Thus,
\begin{align}
    \E_{x}[( (w^1)^Tx - \tilde \alpha^Ts)^2] &= \E_{s,\xi} [ \Big( \gamma \tau  \sum_{i=1}^N \alpha_i s_i \sum_{j \in C_i} m_j (\xi_j - m_j) + \sum_{j=1}^d (\gamma \tau \eta_j' +w_j^0) x_j \Big)^2 ] \\
    & \leq 2 \Big(\gamma^2 \tau^2 \E_{s,\xi}\Big[   \Big( \underbrace{\sum_{i=1}^N \alpha_i s_i \sum_{j \in C_i} m_j (\xi_j - m_j)}_{:=A}\Big)^2\Big] + \E_{s,\xi}\Big[ \Big( \underbrace{\sum_{j=1}^d (\gamma \tau \eta_j' +w_j^0) x_j }_{:=B}\Big)^2   \Big]\Big).
\end{align}

For the first term,
\begin{align}
   \E_{\xi}[A^2] 
   &= \E\!\left[\Big(\sum_{i=1}^N \alpha_i s_i \sum_{j\in C_i} m_j (\xi_j-m_j)\Big)^2\right] \\
   & = \sum_{i,k=1}^N \alpha_i\alpha_k\, m_{C_i}^\top \Sigma_{C_i,C_k}\, m_{C_k}\\
   & \leq \| \Sigma \|_{\rm op} \sum_{i} \alpha_i^2 \| m_{C_i}\|_2^2 \\
   &\leq \| \Sigma \|_{\rm op} \| \alpha \|_{\infty}^2 v_{\rm sum},
\end{align}
where $m_{C_i}=(m_j)_{j\in C_i}$ and $\Sigma_{C_i,C_k}$ denotes the corresponding block 
of $\Sigma$.

For the second term, write $B=\sum_{j=1}^d \eta'' x_j$, where $\eta_j'' := \gamma \tau  \eta_j' +w_j^0$ and note that $\Cov(x)=\Sigma$. Then
\begin{align}
   \E_{s,\xi}[B^2] 
   &= \eta''^\top \Sigma\, \eta'' + (\eta''^\top \E[x])^2 \\
   &\le \|\Sigma\|_{\rm op}\,\|\eta''\|_2^2 ,
\end{align}
since $s$ is symmetric, thus $\E[x]=0$.
Note that $\| w^0\|_2^2 = \theta_d(1)$ (see Lemma~\ref{lem:preliminaries}). Set $\gamma = 1/\sqrt{v_{\rm sum}}$. If $B\geq C \log(dn) M_{f,N}^2 d/v_{\rm sum} $, then $\zeta = O(\sqrt{v_{\rm sum}}/\sqrt{d}) $. Furthermore, if $\tau = O(\frac{\sqrt{v_{\rm sum}}}{\sqrt{d}\log(d)})$, then $ \omega = O(\frac{v_{\rm sum}}{\log(d)d})$. Thus, $\|\eta''\|_2^2=\theta_d(1)$.
Then,
\begin{align}
    \E_{x}[( (w^1)^Tx - \tilde \alpha^Ts)^2] & \leq C_0 
\end{align}

Then, since by Lemma~\ref{lem:certificate} $\| a^*\|_1 \leq \sqrt{n} \|a^*\|_2 \leq \frac{C'}{v_{\min}^{P}}  $, we have
\begin{align}
    \E_{x,y}[| y - \NN(x;\hat \theta)|] & \leq C \|a^*\|_{1}\cdot \sqrt{\E(1+|t|^{P-1})^2}\\
    & \leq C \frac{1}{v_{\min}^P} v_{\min}^{P-1} = C/v_{\min} \leq \delta,
\end{align}
for large enough $d$, since $v_{\min} = \theta(d)$ by assumption.
\end{proof}

\noindent 
To conclude, we use the following well known result on convergence of SGD on convex losses.
\begin{theorem}[\cite{shalev2014understanding}] \label{thm:convergence_SGD_martingale}
    Let $\cL$ be a convex function and let $a^* \in \argmin_{\|a\|_2 \leq \cB} \cL(a) $, for some $\cB>0$. For all $t$, let $\alpha^{t}$ be such that $\E\left[ \alpha^{t} \mid a^{t} \right] =  -\nabla_{a^{t}} \cL(a^{t}) $ and assume $\| \alpha^{t} \|_2 \leq \xi $ for some $\xi>0$. If $a^{(0)} = 0 $ and for all $t \in [T]$ $a^{t+1} = a^{t} +\gamma \alpha^{t} $, with $\gamma = \frac{\cB}{\xi \sqrt{T}}$, then
    \begin{align*}
        \frac{1}{T} \sum_{t=1}^T \cL(a^{t}) \leq \cL(a^*) + \frac{\cB \xi}{\sqrt{T}}.
    \end{align*}
\end{theorem}


\noindent 
We take $\cB =\theta(v_{\rm min}^{-P}), \xi = \theta(n v_{\min}^{P})$, $T_2= \theta(n^2/\epsilon^2)$, and $\gamma_2 = \theta(v_{\rm min}^{-2P} \epsilon/n) $.

\section{Technical Lemmas}
\label{app:technical lemmas}
\begin{lemma} \label{lem:preliminaries}
    Let $w_j^h \overset{iid}{\sim} \cN(0,1/d)$, for $j \in [d], h \in [n]$. Let $\Delta_m(w^h) \in \bR^N$ be such that $(\Delta_m(w^h))_i = \sum_{j \in C_i} w_j m_j$. Then, for any $\delta>0$, with probability $1-2n\delta$, where $n$ is the number of hidden units, and for $d$ large enough, we have:
    \begin{itemize}
        \item[(i)] $\max_{h \in [n]} \|w^h\|_2 \leq 3/2$.
        \item[(ii)] $\max_{h \in [n]} \|\Delta_m(w^h)\|_{L_1} \leq C \sqrt{\log(2/\delta)} $.
        \item[(iii)] $\max_{h,j} |w_j^h| \leq O(\sqrt{\log(dn)}/\sqrt{d})$.
    \end{itemize}
\end{lemma}

\begin{proof}

(i) For each $h$, we have $\|w^h\|_2^2=(1/d)\sum_{j=1}^d Z_j^2=(1/d)\chi^2_d$. By standard $\chi^2$ concentration, $\Pr(\|w^h\|_2\ge 3/2)\le e^{-cd}$. Taking a union bound over $h\in[n]$ gives $\Pr(\max_h \|w^h\|_2 \le 3/2)\ge 1-ne^{-cd}$, which is at least $1-n\delta$ for $d$ large enough so that $e^{-cd}\le\delta$.

(ii) For each $i$, $(\Delta_m(w^h))_i=\sum_{j\in C_i} m_j w_j^h \sim \cN(0,(1/d)\sum_{j\in C_i} m_j^2)$, hence it is sub-Gaussian with variance proxy at most $C^2 = \max_i \| (m_j)_{j\in C_i}\|_2^2/d$. Note that $0<C^2<\infty$ by assumption. Therefore $\Pr(|(\Delta_m(w^h))_i|>t)\le 2e^{-t^2/(2C^2)}$. Choosing $t=C\sqrt{2\log(2/\delta)}$ yields probability at most $\delta$. By a union bound over $h$, we obtain $\max_h |(\Delta_m(w^h))_i|\le C\sqrt{2\log(2/\delta)}$ with probability at least $1-n\delta$. 

(iii) Each coordinate satisfies $\Pr(|w_j^h|>t)\le 2e^{-dt^2/2}$. A union bound over all $dn$ entries gives $\Pr(\max_{h,j}|w_j^h|>t)\le 2dn\,e^{-dt^2/2}$. Taking $t=\sqrt{\tfrac{2}{d}\log\!\tfrac{2dn}{\delta}}$ shows that $\max_{h,j}|w_j^h|\le \sqrt{(2/d)\log(2dn/\delta)}$ with probability at least $1-\delta$, i.e.\ $\max_{h,j}|w_j^h|=O(\sqrt{\log(dn)}/\sqrt d)$.

Combining the three parts and adjusting constants so that $e^{-cd}\le\delta$ shows that all items hold simultaneously with probability at least $1-2n\delta$.
\end{proof}


\begin{lemma}[Concentration of $V$]
\label{lem:V-concentration}
Let $w=(w_1,\dots,w_d)\sim \cN(0,I_d/d)$ be independent of the noise vector $\xi$, and suppose the coordinates of $\xi$ are jointly sub-Gaussian with covariance matrix $\Cov(\xi)=\Sigma\succeq 0$. Define
\[
V \;=\; w^\top \Sigma w \;=\; \frac{1}{d}\,Z^\top \Sigma Z, \qquad Z\sim\cN(0,I_d).
\]
Then $\E_w[V]=\mu:=\tfrac{1}{d}\tr(\Sigma)$. Moreover, for some universal constant $c>0$ and every $\epsilon\in(0,1]$,
\begin{equation}\label{eq:HW-relative-general}
\Pr\!\left(\,|V-\mu|\ge \epsilon \mu\,\right)
\;\le\;
2\exp\!\Big(-\,c\,\frac{d\,\mu}{\|\Sigma\|_{\mathrm{op}}}\,\epsilon^2\Big).
\end{equation}
\end{lemma}

\begin{proof}
By the Hanson--Wright inequality for Gaussian quadratic forms,
\[
\Pr\!\left(\left|Z^\top \Sigma Z - \tr(\Sigma)\right|\ge t\right)
\;\le\; 2\exp\!\left(-c\,\min\!\left\{\frac{t^2}{\|\Sigma\|_F^2},\ \frac{t}{\|\Sigma\|_{\mathrm{op}}}\right\}\right).
\]
Substituting $t=d\,\varepsilon \mu=\varepsilon \tr(\Sigma)$ and dividing by $d$ gives
\[
\Pr\!\left(\,|V-\mu|\ge \varepsilon \mu\,\right)
\;\le\; 2\exp\!\left(-c\,\min\!\left\{\frac{d^2 \varepsilon^2 \mu^2}{\|\Sigma\|_F^2},\ \frac{d\,\varepsilon \mu}{\|\Sigma\|_{\mathrm{op}}}\right\}\right).
\]
Since $\|\Sigma\|_F^2 \le \|\Sigma\|_{\mathrm{op}}\tr(\Sigma) = d\,\mu\,\|\Sigma\|_{\mathrm{op}}$, we obtain \eqref{eq:HW-relative-general}.
\end{proof}

\begin{lemma}[Stability under random variance]
\label{lem:activation-stability}
Let Assumption~\ref{ass:activation} hold, and let
\(
V \;=\; w^\top \Sigma w, \qquad w\sim\cN(0,I_d/d),
\)
independent of $\xi$, with $\mu=\tfrac{1}{d}\tr(\Sigma)$.  
Define, for $s\ge 0$,
\[
S_s(t) \;:=\; \E_{G\sim\cN(0,s)}[\sigma'(t+G)].
\]
Then there exists a constant $\delta>0$ (depending only on $\sigma$, $N$, and $\mu+v_{\rm sum}/d$) such that, on the event
\(
|V-\mu|\le \delta,
\)
the following holds simultaneously for all $k\le N$:
\[
|S_{V+v_{\rm sum}/d}^{(k)}(0) |\;\ge\; \frac{c_0}{2},
\]
with probability at least
\(
1-2\exp\!\Big(-c\,\frac{d\,\mu}{\|\Sigma\|_{\mathrm{op}}}\,\frac{\delta^2}{\mu^2}\Big).
\)

\end{lemma}

\begin{proof}
For $s\ge 0$ and $k\ge 0$, differentiation under the expectation is justified since $\sigma$ is a polynomial, which gives
\[
S_s^{(k)}(0)
\;=\;
\E_{G\sim\cN(0,s)}[\sigma^{(k+1)}(G)].
\]
Hence $S_s^{(k)}(0)$ depends on $s$ only through the Gaussian variance.

The map $s\mapsto S_s^{(k)}(0)$ is continuously differentiable, and for $G_s\sim\cN(0,s)$ we have the identity
\[
\frac{d}{ds} S_s^{(k)}(0)
\;=\;
\frac12\,\E_{G_s}\!\left[\sigma^{(k+3)}(G_s)\right].
\]
Since $\sigma^{(k+3)}$ is a polynomial, the right-hand side is finite and continuous in $s$. Therefore, for each $k\le N$, there exists a constant $L_k<\infty$ such that
\[
\big|S_s^{(k)}(0)-S_{s_0}^{(k)}(0)\big|
\;\le\;
L_k\,|s-s_0|,
\qquad
s_0:=\mu+v_{\rm sum}/d,
\]
for all $s$ in a neighborhood of $s_0$. Let $L:=\max_{k\le N} L_k$.

By Assumption~\ref{ass:activation}, $S_{s_0}^{(k)}(0)\ge c_0$ for all $k\le N$. Choosing
\(
\delta \;:=\; \frac{c_0}{2L},
\)
we obtain that whenever $|s-s_0|\le \delta$,
\[
S_s^{(k)}(0)
\;\ge\;
S_{s_0}^{(k)}(0)-L|s-s_0|
\;\ge\;
\frac{c_0}{2},
\qquad
\forall\,k\le N.
\]

Finally, since $s=V+v_{\rm sum}/d$ and $s_0=\mu+v_{\rm sum}/d$, the condition $|s-s_0|\le\delta$ is equivalent to $|V-\mu|\le\delta$. The probability bound follows directly from Lemma~\ref{lem:V-concentration}.
\end{proof}

\begin{remark}
Lemma~\ref{lem:V-concentration} above shows that $V$ concentrates at scale
$\sqrt{(C/(d\mu))}$ around $\mu$. In particular, for fixed $C$ and $\mu$ bounded away from $0$, the dependence on the specific initialization vanishes exponentially fast in $d$.
\end{remark}

\subsection{Proof of Lemma~\ref{lem:pop_grad}}
\label{app:proof_pop_grad}

Fix $j\in C_i$. Write $a_k:=s_{c(k)}w_k$ and $X_k:=\xi_k-m_k$. Define
\[
\mu := \sum_{k=1}^d a_k m_k = s^\top \Delta_m(w),
\quad
U := \sum_{k=1}^d a_k X_k,
\quad
U_{-j} := \sum_{k\ne j} a_k X_k.
\]
By independence of the $\xi_k$, conditional on $s$,
$\Var(U\mid s) = \sum_k w_k^2 \sigma_k^2 = V^2$, which does not depend on $s$.

Let $\cE_w$ denote the high-probability event of Lemma~\ref{lem:preliminaries}
together with the chi-squared concentration estimate
$\|\Delta_m(w)\|_2^2 \le 2\,v_{\rm sum}/d \le C$
(which holds with probability $1-e^{-cd}$ since
$\E\|\Delta_m(w)\|_2^2 = v_{\rm sum}/d = O_d(1)$).
On $\cE_w$:
\[
\max_k |w_k| \le C\sqrt{\log d}/\sqrt d,
\quad
\|w\|_2 \le 3/2,
\quad
\|\Delta_m(w)\|_2 \le C',
\quad
V^2 \asymp 1.
\]
In particular, for every $s \in \{\pm 1\}^N$,
\[
|\mu| = |s^\top \Delta_m(w)| \le \sqrt{N}\,\|\Delta_m(w)\|_2 \le C\sqrt N = O_d(1),
\]
since $N = O_d(1)$. Throughout the proof we work on $\cE_w$ and the constants
$C_1,C_2,C_3$ depend only on $\sigma$, on the sub-Gaussian parameters and third
moments of $\xi$, and on $N$ — but not on $d$.

Writing $\xi_j = m_j + X_j$,
\begin{equation}\label{eq:popgrad-split}
\E_\xi[\xi_j\,\sigma'(\mu+U)]
\;=\; m_j\,\E_\xi[\sigma'(\mu+U)]
\;+\; \E_\xi[X_j\,\sigma'(\mu+U)].
\end{equation}

Since $\sigma'$ is a polynomial of degree $P-1$, expand $\sigma'(\mu+x) = \sum_{r=0}^{P-1} c_r(\mu)\,x^r$ with $|c_r(\mu)|\le C(1+|\mu|^{P-1-r}) = O_d(1)$ on $\cE_w$. Then
\[
\E_\xi[\sigma'(\mu+U)] - \E_G[\sigma'(\mu+G)]
\;=\; \sum_{r=3}^{P-1} c_r(\mu)\,\big(\E[U^r] - \E[G^r]\big),
\]
since the $r=0,1,2$ moments match. By the standard moment-comparison bound for sums of independent centered sub-Gaussian variables,
\[
\big|\E[U^r] - \E[G^r]\big| \;\le\; C_r\,\sum_k |a_k|^3\,\E|X_k|^3
\;\le\; C\,\frac{\sqrt{\log d}}{\sqrt d},
\]
using $\sum_k |a_k|^3 \le \max_k|a_k|\cdot\|w\|_2^2 \le C\sqrt{\log d}/\sqrt d$ on $\cE_w$ and bounded third moments. Summing over $r$,
\begin{equation}\label{eq:BE_step}
\big|\E_\xi[\sigma'(\mu+U)] - \E_G[\sigma'(\mu+G)]\big|
\;\le\; C_\sigma\,\varepsilon_d,
\qquad
\varepsilon_d \;:=\; \frac{\sqrt{\log d}}{\sqrt d}.
\end{equation}

Since $\xi_j \perp \xi_{-j}$, we have $X_j \perp U_{-j}$. Taylor-expand
$\sigma'$ around $\mu + U_{-j}$:
\[
\sigma'(\mu+U_{-j}+a_jX_j)
\;=\; \sigma'(\mu+U_{-j})
+ a_j X_j\,\sigma''(\mu+U_{-j})
+ \tfrac{(a_jX_j)^2}{2}\,\sigma^{(3)}(\zeta_j),
\]
for some intermediate point $\zeta_j$ between $\mu+U_{-j}$ and $\mu+U_{-j}+a_jX_j$.
Multiplying by $X_j$ and taking expectation over $\xi$, then using
$X_j \perp U_{-j}$ and $\E[X_j]=0$,
\[
\E_\xi[X_j\,\sigma'(\mu+U)]
\;=\; a_j\,\sigma_j^2\,\E[\sigma''(\mu+U_{-j})]
\;+\; \tfrac{a_j^2}{2}\,\E\!\big[X_j^3\,\sigma^{(3)}(\zeta_j)\big].
\]
Each $\sigma^{(k)}(\mu+U_{-j})$ is a polynomial of degree $P-k$ in $U_{-j}$
with coefficients that are $O_d(1)$ on $\cE_w$ (since $|\mu|=O_d(1)$). As
$U_{-j}$ is sub-Gaussian with variance $V_{-j}^2 \asymp 1$, all its moments
are $O_d(1)$. Hence
\[
\big|\E[\sigma''(\mu+U_{-j})]\big| \;\le\; C_\sigma,
\qquad
\big|\E[X_j^3\,\sigma^{(3)}(\zeta_j)]\big|
\;\le\; (\E X_j^6)^{1/2}\,(\E\sigma^{(3)}(\zeta_j)^2)^{1/2}
\;\le\; C_\sigma',
\]
where the second bound uses Cauchy--Schwarz, bounded sub-Gaussian moments of
$X_j$, and the fact that $\zeta_j$ is sub-Gaussian (being bounded by two
sub-Gaussian random variables). Therefore
\begin{equation}\label{eq:secondterm}
\big|\E_\xi[X_j\,\sigma'(\mu+U)]\big|
\;\le\; C_1\,|w_j| + C_2\,w_j^2.
\end{equation}

Multiplying \eqref{eq:popgrad-split} by $f(s)\,s_i$ and averaging over $s$,
\[
\overline G(w_j)
\;=\; m_j\,\E_s[f(s)\,s_i\,\E_\xi[\sigma'(\mu+U)]]
\;+\; \E_s[f(s)\,s_i\,\E_\xi[X_j\,\sigma'(\mu+U)]].
\]
Applying \eqref{eq:BE_step} inside the first $s$-expectation and
\eqref{eq:secondterm} for the second, both bounds being uniform in $s \in
\{\pm 1\}^N$ on $\cE_w$, we conclude
\[
\overline G(w_j) \;=\; m_j\,\alpha_i(w) + \eta_j,
\qquad
|\eta_j| \;\le\; C_1\,|w_j| + C_2\,w_j^2 + C_3\,\varepsilon_d. \qed
\]

\section{Proof of Theorem~\ref{thm:homogeneous_clusters}}
\label{app:proof_homogeneous}
The proof of Theorem~\ref{thm:homogeneous_clusters} follows the same path as the proof of Theorem~\ref{thm:clusters}. The main differences are in the computation of the initial population gradient (derived in Lemma~\ref{lem:pop_grad}, which is replaced by Lemma~\ref{lem:bound_p} below), and in the construction of the certificate (Lemma~\ref{lem:certificate}, replaced by Lemma~\ref{lem:certificate_homogeneous}). Moreover, the non-degeneracy (derived in the previous proof in Lemma~\ref{lem:non_degeneracy}) is no longer derived from the randomness of the initialization, and it is instead assumed (Assumption~\ref{ass:nondegeneracy_f}). 

\paragraph{Initial population gradient.} In this Section, we consider learning with a 2-layer network with ReLU activation: $\NN(x;\theta^t) = \sum_{h=1}^n a_h \ReLU((w_h^t)^\top x + b_h)$. We assume that the target function satisfies Assumption~\ref{ass:nondegeneracy_f}. We initialize $w_{hj}^{(0)}=1/\sqrt{d}$ (deterministically) and $a_h^{(0)}=\tau$, with $\tau>0$, specified later. Similarly to the previous proof, we initialize the bias neurons to $0$ and keep them frozen during the first phase of training. Before the second phase of training, we draw $b_h \overset{iid}{\sim} \Unif [-A,A]$, for a fixed $A>0$, specified later.
Because $w_h^0 =1/\sqrt{d}$ for all $h \in [n]$, all hidden neurons behave analogously. Thus, in the following we consider a single vector $w \in \bR^d$, dropping the subscript $h$ and upscript $0$ for simplicity. For simplicity, we assume $N$ to be odd. Let $\bar \xi = (\bar \xi_i)_{i \in [N]}$, where $ \bar \xi_i =\frac{1}{k} \sum_{j \in C_i} \xi_{j}$. For $j \in C_i$, let us compute the initial population gradient.

\begin{align}
\bar G(w_{j}):&= \E_{x,y} [ y x_{j} \sigma'(wx)] \\
& = \E_{s,\xi} \left[ f(s) s_i \xi_{j} \sigma'\left( s^T \bar \xi k/\sqrt{d} \right)\right] \\
& = \E_{s} \left[ f(s) s_i \E_{\xi|s} \left[ \xi_{j} \sigma'\left( s^T \bar \xi k/\sqrt{d} \right)\right] \right].
\end{align}
Let us consider the inner expectation. Recall that we assumed $\ReLU$ activation, thus $\sigma'(x) =  \mathds{1}(x \geq 0)$. Thus,
\begin{align} \label{eq:p1_pm1}
    \E_{\xi|s} [\xi_{j} \mathds{1}\left( s^T \bar \xi \geq 0 \right) ] & = (1-\delta) p_{1} - \delta p_{-1},
\end{align}
where, for $b \in \{ \pm 1 \}$, $p_b : = \pr ( s^T \bar \xi \geq 0 \mid \xi_{j} = b)$. We make use of the following Lemma.

\begin{lemma} \label{lem:bound_p}
    For $b \in \{ \pm 1 \}$, 
    \begin{align}
       \left| p_{b} - \mathds{1}((1-2\delta) s^T 1 > 0) \right| \leq \exp(-C k N (1-2\delta)^2 \bar s^2),
    \end{align}
    where $C>0$ is a constant and $\bar s :=\frac{1}{N}\sum_{i=1}^N s_i $.
\end{lemma}
\begin{proof}[Proof of Lemma~\ref{lem:bound_p}]
    First, let us write $ s^T \bar \xi = s_i \xi_{j} + S_{-j}$, where $S_{-ij}: = \sum_{(l,h) \neq (i,j)} s_l \xi_{h}$. Let $\mu: = k (1-2\delta) s^T 1 - s_i(1-2\delta)  $ 
    denote the mean of $ S_{-ij}$.  
    Note that $ p_b = \pr(S_{-ij} \geq - b \cdot s_i  )$. Moreover, because we assumed $N$ to be odd, $s^T 1 \neq 0$. Assume first that $(1-2\delta)s^T 1 <0$. Then, by Hoeffding's inequality,
    \begin{align}
        p_b & = \pr(S_{-ij} \geq - b \cdot s_i  )\\
        & \leq \pr ( S_{-ij} - \mu \geq - \mu/2 ) \\
        & \leq \exp\left( -\frac{k N (1-2\delta)^2 \bar s^2}{8} \right).
    \end{align}
 On the other hand, if $ (1-2\delta) s^T 1 >0$,
\begin{align}
    p_b & =1- \pr(S_{-ij}  \leq - b \cdot s_i  ) \\
    & \leq \pr ( \mu - S_{-ij} \geq \mu/2 ) \\
    & \leq \exp\left( -\frac{k N (1-2\delta)^2 \bar s^2}{8} \right).
\end{align}
\end{proof}

\noindent
Thus, applying Lemma~\ref{lem:bound_p} to~\eqref{eq:p1_pm1},
\begin{align}
    \E_{\xi|s} [\xi_{j} \mathds{1}\left( s^T \bar \xi \geq 0 \right) ] = (1-2\delta) \mathds{1}((1-2\delta) s^T1 >0) + \eta_s,
\end{align}
where $\eta_s$ is an error term such that $|\eta_s| \leq \exp(-C k N (1-2\delta)^2 \bar s^2)\leq \exp(-C k (1-2\delta)^2/N )$, since $|\bar s|\geq 1/N $, because $N$ is odd. Note that since by assumption $d(1-2\delta)^2 =\Omega(\frac{\log(d)}{1-2\delta})$, we have $|\eta_s| =\Omega(d^{-c})$, for some $c>0$. Plugging this into the outer expectation over $s$, we obtain, using $\sign(1-2\delta) =1$, since $\delta \in (0,1/2)$:
\begin{align*}
   \bar G(w_j) & = \E_s\left[f(s) s_i \E_{\xi|s} [\xi_{j} \mathds{1}\left( s^T \bar \xi \geq 0 \right) ]\right]\\
   & = (1-2\delta)\E[f(s) s_i \mathds{1}((1-2\delta)s^T 1 >0) ] + \eta' \\
   & = \frac{1}{2}(1-2\delta) \left( \E[f(s) s_i] + \sign(1-2\delta)\E[f(s) s_i \Maj(s)] \right)+ \eta'\\
   & = \frac{1-2\delta}{2} \hat f(\{ i \}) + \frac{1}{2}\Big(\sum_{T: i \in T} \hat f(T) \hat \Maj(|T|-1) +\sum_{T : i \not\in T} \hat f(T) \hat \Maj(|T|+1) \Big) + \eta'\\
   & = \frac 12 (1-2 \delta) \sum_{T \subseteq [N]} \hat f(T) c_{T,i} +\eta'\\
   & = (1-2\delta) \alpha_i + \eta',
\end{align*}
where, $|\eta'| = |\E_s[ f(s) s_i \eta] \leq \E_s[|f(s)|] \cdot \sup_s |\eta_s| $, $c_{T,i}$ are defined in~\eqref{eq:cTi} and where we defined $\alpha_i := \frac 12 \sum_{T \subseteq [N]} \hat f(T) c_{T,i} $.




\paragraph{Initial estimated gradient.} Let us now consider the gradient estimated through $B$ i.i.d. samples.
\begin{align}
    G_B(w) :& = \frac 1B \sum_{b=1}^B y^b \sigma'(w x^b)x^b .
\end{align}
If $B$ is large enough, the estimated gradient is close to the population gradient $\bar G(w_h^0)$, as formalized by the following Lemma.  
\begin{lemma} \label{lem:estimated_G}
    Let us denote by $M_{f,N} = \max_{s \in \{ \pm 1 \}^N} |f(s)|$. If the batch size $B \geq C \log(dn) M_{f,N}^2 /2 \zeta^2$, for $\zeta,C>0$, then with probability $1-d^{-C^*}$, for $C^*>0$, for all $ h \in [n]$,
    \begin{align}
        \| G_B(w_{h}^0) - \bar G(w_{h}^0) \|_{\infty} \leq \zeta.
    \end{align}
\end{lemma}
\begin{proof}
Let us notice that for a $j \in [d]$, and for $\sigma = \ReLU$ activation,
\begin{align}
    |G_B(w_h^0)_j| 
    & \leq M_{f,N}.
\end{align}
By Hoeffding's inequality, there exists $C^*>0$ such that with probability $1-d^{-C^*}/dn$,
\begin{align}
    \Big| G_B(w_h^0)_j- \bar G(w_h^0)_j  \Big| \leq \zeta.
\end{align}
The result follows by union bound.
\end{proof}

\begin{lemma}
    Let $z : = w^1 x $. Assume $ |1-2\delta| = \Omega(\log(d)^2/\sqrt{d})$. Assume $B \geq C \log(dn) M_{f,N}^2/2 \zeta^2$, with $\zeta =\theta(\frac{1-2\delta}{\log(d)})$. Assume the initialization scale $\tau =\theta(\frac{(1-2\delta)^2}{\log(d)d})$, the learning rate $\gamma =\theta(\frac{\log(d)}{(1-2\delta)^4})$ and $N=\theta(1)$. Then, with probability $1-o_d(1)$,
    \begin{align}
       \E_{x}| z - \frac{1}{N} \sum_{i=1}^N \alpha_i s_i | = O\left(\frac{1}{\log(d)}\right).
    \end{align}
\end{lemma}
\begin{proof}
For $j \in C_i$, we have:
    \begin{align}
        w^{1}_j &= w^0_j +\gamma \tau\bar G(w^0) + \gamma \tau( G(w^0) - \bar G(w^0)) + O \left(\gamma \tau^2 d M_{f,N} \right)\\
        & \overset{(a)}{=}  \tau \gamma (1-2\delta) \alpha_i + \gamma \tau (\eta' +\zeta+ 1/\sqrt{d} + O( \tau d M_{f,N}))\\
        & \overset{(b)}{=} \tau \gamma \Big((1-2\delta) \alpha_i + \eta''\Big),
    \end{align}
where in $(a)$ we used the assumption on the initialization and Lemma~\ref{lem:estimated_G}, and in $(b)$ we denote $\eta'':=\eta' +\zeta+ 1/\sqrt{d} + O( \tau d M_{f,N}) $, and note that by the assumptions on the batch size, learning rate and initialization scale:
\begin{align}
    \eta'' = O\Big( \max\Big\{ \frac{|1-2\delta|}{\log(d)}, d^{-1/2}\Big\} \Big) = O\Big(\frac{|1-2\delta|}{\log(d)}\Big) .
\end{align}

Let $\mathcal E_\xi$ be the event on which, for all $i\in[N]$,
\begin{align}
\Big|\sum_{l=1}^k \xi_{il} - k(1-2\delta)\Big| \le \sqrt{k}\,\log d .
\end{align}
By Hoeffding's inequality and a union bound over $i$, $\pr(\mathcal E_\xi)\ge 1-d^{-C}$
for some constant $C>0$. 

On $\mathcal E_\xi$, using the expression for $w^1_j$ derived above, we have
\begin{align}
z = (w^1)^\top x
&= \gamma\tau\Big((1-2\delta)\sum_{i=1}^N \alpha_i s_i \sum_{j\in C_i}\xi_j
\;+\; d\,\eta''\Big).
\end{align}
Write
\[
\sum_{j\in C_i}\xi_j = k(1-2\delta) + \Delta_i,
\qquad |\Delta_i|\le \sqrt{k}\log d ,
\]
which holds on $\mathcal E_\xi$. Substituting gives
\begin{align}
z
&= \gamma\tau\Big(
(1-2\delta)^2 k \sum_{i=1}^N \alpha_i s_i
+ (1-2\delta)\sum_{i=1}^N \alpha_i s_i\,\Delta_i
+ d\,\eta''
\Big)\\
& \overset{(a)}{=} c_0\cdot \frac{1}{N}\sum_{i=1}^N \alpha_i s_i
+ \gamma\tau(1-2\delta)\sum_{i=1}^N \alpha_i s_i\,\Delta_i
+ \gamma\tau\, d\,\eta'',
\end{align}
for some $c_0>0$, where in $(a)$ we used the assumptions on $\gamma,\tau$ and $k =d/N$.
We now bound the remainder terms. 
\begin{align}
\Big|\gamma\tau(1-2\delta)\sum_{i=1}^N \alpha_i s_i\,\Delta_i\Big|
&\le O\!\left(\gamma\tau\,|1-2\delta|\,\sqrt d\,\log d\right)\\
& \leq O\!\left(\frac{\log(d)}{\sqrt{d}|1-2\delta|}\right),
\end{align}
where the second inequality uses $\gamma\tau=\Theta((1-2\delta)^{-2}d^{-1})$.
Using the assumption $|1-2\delta|=\Omega(\log (d)^2/\sqrt d)$, this is
$O(1/\log d)$.

Next, recall the bound $\eta'' = O(|1-2\delta|/\log(d)$, which implies
\begin{align}
|\gamma\tau\, d\,\eta''|
= O\!\left(
\gamma\tau d\frac{|1-2\delta|}{\log d}\right) = O\left(\frac{|1-2\delta|}{\log(d)}
\right).
\end{align}
Since
$|1-2\delta|\leq 1$, this yields
$|\gamma\tau\, d\,\eta''| = O(1/\log d)$.

Combining the above bounds, on the event $\mathcal E_\xi$ we obtain
\[
\Big| z - \frac{1}{N}\sum_{i=1}^N \alpha_i s_i \Big|
= O\!\left(\frac{1}{\log d}\right).
\]
Since $\pr(\mathcal E)=1-o_d(1)$, taking expectation over a fresh draw of $x$
gives
\[
\E_x\Big| z - \frac{1}{N}\sum_{i=1}^N \alpha_i s_i \Big|
= O\!\left(\frac{1}{\log d}\right).
\]
\end{proof}

\begin{lemma} \label{lem:certificate_homogeneous}
    Let $f$ satisfy the non-degeneracy Assumption~\ref{ass:nondegeneracy_f} and let the biases be sampled as $b_h \overset{iid}{\sim} \Unif[-A,A] $, with $A :=2N 2^{N/2} \sqrt{\Var(f)}$, for $h \in [n]$. Then, there exist a constant $\Delta>0$ such that if $n\geq \frac{C \log(d) A}{\Delta N} $, there exists 
    a vector $a^* \in \bR^{n}$ such that $\| a^* \|_{\infty} \leq \frac{2}{\Delta} M_{f,N} $ and
    \begin{align} \label{eq:astar}
        f(s) = \sum_{h=1}^n  a^*_h \sigma( s^T \alpha + b_h), 
    \end{align}
    for all $s \in \{ \pm 1 \}^N$.
    \end{lemma}

\begin{proof}
    Let us first enumerate the vectors $s_l \in \{ \pm 1 \}^N$, $l \in [2^N]$, so that the real numbers 
    \begin{align}
        v_l = \sum_{i=1}^N s_{l,i} \alpha_i
    \end{align}
    are in strictly increasing order, i.e. $v_1 < v_2<...<v_{2^N}$. This is possible since by Assumption~\ref{ass:nondegeneracy_f} the map $s \mapsto \sum_i s_i \alpha_i$ is injective on $\{ \pm 1 \}^N$. 

    Let $\Delta = \min_{1 \leq l <2^N} (v_{l+1}-v_l) >0$. Fix any $\delta \in (0,\Delta/2)$ and consider the interval $I= [-v_{2^N} -\delta, - v_1 +\delta ] $. Note that $I \subseteq [-A,A].$ Partition $I$ into $2^N$ disjoint sub-intervals of width $\Delta/2$: $I_j = [-v_j + \delta/2, -v_j + \delta] $. Now, if $n$ biases are sampled uniformly from $[-A,A]$, for each $j\in [2^N]$, the probability that none of the $n$ samples falls into $I_j$ is 
    \begin{align}
        \left( 1- \frac{\Delta/2}{2A} \right)^{n} \leq \exp\left(- \frac{\Delta}{2A}n \right).
    \end{align}
    By union bound, the probability that every interval $I_j$ is hit at least once is 
    \begin{align}
        1- 2^N \exp\left(- \frac{\Delta}{2A}n \right) \geq 1-d^{-C'},
    \end{align}
    if $n \geq \frac{C \log(d) A}{\Delta N}$, for suitable constants $C,C'>0$. For each $j \in [2^N]$, choose one $b_{h_j} $ such that $b_{h_j}  \in I_j$, and denote by $J = \{ h_j:j\in [2^N]\} \subset [n]$ such set. Let $M \in \bR^{2^N \times 2^N}$ be such that
    \begin{align}
        M_{l,j} :&= \ReLU(v_l +b_j) \\
        & = v_l-v_j + \epsilon,
    \end{align}
    where $\epsilon \in (0,\Delta/2)$.
    Note that, if $l<j$, then $v_l-v_j<\Delta$, thus $M_{l,j} =0$, while if $ l \geq j$, $v_l-v_j\geq0 $, thus $ M_{l,j} >0$. Thus, $M$ is lower triangular, with strictly positive entries on the diagonal, and thus invertible. Let $c^* \in \bR^{2^N}$ be such that
    \begin{align}
        c^* = F M^{-1},
    \end{align}
    where $F \in \bR^{2^N}$ is a vector such that $F_l = f(s_l)$. Then, $a^* \in \bR^n $ defined by $ a^*_{h_j} = c^*_{j} \mathds{1}(h_j \in J) $ satisfies~\eqref{eq:astar}. Moreover, 
    \begin{align}
        \| a^*\|_{\infty} \leq \| M^{-1}\|_\infty \| F\|_{\infty} \leq \frac{2}{\Delta}\cdot M_{f,N}.
    \end{align}
\end{proof}

From the previous Lemmas, it follows that with probability $1-d^{-C}$, for some $C>0$,
\begin{align}
    \E_{x,y} & \left|y- \sum_{h=1}^n a_h^* \sigma(w^1 x + b_h)\right|   \\
    & \leq \E_{x,y} \left| y- \sum_{h=1}^n a_h^* \sigma(\alpha s + b_h)\right| +\E_{x,y} \left| \sum_{h=1}^n a_h^* ( \sigma(\alpha s + b_h) - \sigma(w^1 x + b_h)) \right|  \\
    & \leq n \| a^*\|_{\infty} \E_x \left| \sigma(\alpha s + b_h) - \sigma(w^1 x + b_h) \right| \\
    & \leq n \| a^*\|_{\infty} \E_x \left| \alpha s  - w^1 x \right|\\
    & \leq O \left( \frac{\log d 2^N}{\Delta N} \frac{1}{\Delta} \frac{1}{\log(d)^2} \right) =  O\left(\frac{2^N}{N\Delta^2 \log(d)}\right).
\end{align}

\noindent 
We then use the well known result of Theorem~\ref{thm:convergence_SGD_martingale}. In particular, we note that $\nabla_{a^t} \NN(x;(w^1,a^t,b)) \leq C \sqrt{n}$, for some $C>0$. Thus, choosing $\cB = \sqrt{n} \cdot \| a^*\|_{\infty} = O\Big(\frac{\sqrt{\log(d)}2^{N/2} }{\Delta^{3/2} \sqrt{N}} \Big)$, $\xi = O\left(\frac{\sqrt{\log(d)}}{(\Delta N)^{1/2}}\right)$, $T_2 = O\left(\frac{\log(d)^2 2^N}{\epsilon^2 \Delta^4 N^2}\right)$ and $\gamma_2 = \theta(\frac{\epsilon \Delta N}{\log(d)})$.

\section{Clustering Algorithms} \label{sec:clustering_algo}

\paragraph{Comparison with classical clustering methods.}
Let us briefly compare Theorem~\ref{thm:homogeneous_clusters} with the guarantees of two classical  unsupervised clustering methods. 
In the homogeneous BSC model with equal cluster sizes $|C_i|=k=d/N$ (Example~\ref{exa:treemodel}), the population covariance matrix takes the form
\begin{align*}
\Sigma=(1-m^2)I_d+m^2\sum_{i=1}^N \mathbf{1}_{C_i}\mathbf{1}_{C_i}^\top,
\end{align*} 
where $m=1-2\delta$ is the signal mean. Thus, $\Sigma$ consists of the identity plus $N$ spikes of size
\begin{align*}
    m^2(k-1)\asymp (1-2\delta)^2(d/N).
\end{align*}
\emph{Spectral clustering} succeeds in weakly recovering the clusters once these spikes cross the Baik-Ben Arous-Péché (BBP) threshold, namely when 
$$\frac{d}{N}(1-2\delta)^2 \gtrsim c$$
for a suitable constant $c$~\cite{baik2005phase,benaych2011eigenvalues}. If $\frac{d}{N}(1-2\delta)^2 = \omega_d(1)$ (i.e. if $v_{\rm sum} = \omega_d(1)$), recovery becomes asymptotically exact, see details below. Our \emph{layerwise-SGD} method attains a vanishing error whenever $v_{\rm sum}=d(1-2\delta)^2 = \Omega_d(\log(d)^2)$, which (for fixed $N$) matches the spectral detectability up to logarithmic terms. We remark that spectral recovery uses $B =\theta_d(d)$ samples, as opposed to the $B=\Omega(\log^2(d))$ required by layerwise SGD. 

On the other hand, a simple \emph{covariance–thresholding} baseline, where one estimates $C_{ij}=\tfrac{1}{B}\sum_{b=1}^B x_i^{(b)}x_j^{(b)}$ and threshold at $\tfrac{1}{2}(1-2\delta)^2$, succeeds with high probability using $B\gtrsim (\log d)/(1-2\delta)^4$ samples (this can be shown by simple Hoeffding and union bounds, see below), which is comparable to layerwise SGD (up to $(1-2\delta)^2$ factors).

\paragraph{Spectral Clustering.}
Let $x^1, \dots, x^m \in \mathbb{R}^d$ be $m \sim \Theta(d)$ independent observations of a $d$-dimensional random vector. Assume the covariance matrix $\Sigma \in \mathbb{R}^{d \times d}$ of the distribution generating these observations takes the form
\[
\Sigma = I_d + \Delta,
\]
where $\Delta$ is a rank-$N$ symmetric perturbation with nonzero eigenvalues $\lambda_1, \dots, \lambda_N$. This setting models $r$ clusters in the data, each inducing a structure in the covariance. Under this model, it is possible to recover the eigenvalues $\lambda_1, \dots, \lambda_N$, as well as their associated eigendirections, provided that
\[
|\lambda_i| \geq \theta,
\]
for a threshold $\theta = \Theta(1)$. In the regime where $|\lambda_i| \gg 1$, recovery becomes asymptotically exact. This justifies the use of spectral methods such as PCA for recovering low-rank structure in high-dimensional data, under appropriate signal-to-noise conditions.

Now consider our model where the data consist of $N$ clusters, each of size $d/N$, and the covariance structure satisfies:
\begin{itemize}
  \item For coordinates $i, j$ within the same cluster: $\mathrm{Cov}(x_i, x_j) = (1 - 2\delta)^2$,
  \item For coordinates in different clusters: $\mathrm{Cov}(x_i, x_j) = 0$.
\end{itemize}
\noindent 
In this case, the rank-$N$ perturbation $\Delta$ has eigenvalues of order
\[
\lambda_i \sim \frac{d}{N}(1 - 2\delta)^2.
\]
To ensure that the perturbation is detectable, it suffices that
\[
\lambda_i \gg 1 \quad \Longleftrightarrow \quad \frac{d}{N}(1 - 2\delta)^2 \gg 1 \quad \Longleftrightarrow \quad (1 - 2\delta)^2 \gg \frac{N}{d}.
\]
This translates to the condition
\[
\left| \delta - \frac{1}{2} \right| \gg \frac{1}{\sqrt{d}}.
\]


\paragraph{`Naive' Clustering.}
Let $x^1, \dots, x^m \in \mathbb{R}^d$ be $m$ independent observations of a $d$-dimensional random vectors sampled from a clustered distribution such that coordinates within each cluster have covariance $(1-2\delta)^2$ and coordinates in different clusters are independent. Then, one could think of cluster the coordinates by estimating the pairwise covariance, i.e.:
\begin{align}
C_{i,j} :=\frac 1 T \sum_{t=1}^T x_i^t x_j^t. 
\end{align}
By Hoeffding's inequality:
\begin{align}
    \pr \left( | C_{i,j} - (1-2\delta)^2 \mathds{1}(c(i)=c(j))| \geq (1-2\delta)^2/2 \right) \leq 2 \exp(-(1-2\delta)^4 T/8) \leq 1/d^2,
\end{align}
if $T \geq 16 \log(d) /(1-2\delta)^4 $.

\section{Deferred Proofs}
\label{app:deferred_proofs}

\begin{lemma}[Quantitative nondegeneracy under random activation coefficients]\label{lem:quantitative-beta}
Fix $a>0$ (e.g.\ $a=v_{\rm sum}$) and $\mu>0$. Expand $\sigma'$ in the Hermite basis aligned with variance $a$,
\[
\sigma'(x)\;=\;\sum_{n=0}^{M} \alpha_n\,\He^{(a)}_n(x),
\qquad
\He^{(a)}_n(x):=a^{n/2}H_n\!\big(x/\sqrt{a}\big),
\]
where $H_n$ are the probabilists' Hermite polynomials. Assume the coefficients are independent Gaussians
$\alpha_0,\ldots,\alpha_M \overset{\text{i.i.d.}}{\sim} \cN(0,1)$ and $M\ge N$.
Let
\[
\psi_\mu(x)\;:=\;\E_{G\sim\cN(0,\mu)}\big[\sigma'(x+G)\big],
\qquad
\beta_k(\mu)\;:=\;\frac{1}{k!}\,\E_{X\sim\cN(0,1)}\!\big[\psi_\mu(\sqrt{a}\,X)\,H_k(X)\big],
\quad 0\le k\le N.
\]
Then each $\beta_k(\mu)$ is (marginally) Gaussian, $\beta_k(\mu)\sim \cN(0,s_k^2)$, with variance
\begin{equation}\label{eq:sk2}
s_k^2 \;=\; a^{k}\sum_{r=0}^{\lfloor (M-k)/2\rfloor}
\Bigg[\frac{(k+2r)!}{k!}\,\frac{(\mu/2)^r}{r!}\Bigg]^2 \;>\;0.
\end{equation}
In particular, for any $\tau>0$,
\begin{equation}\label{eq:small-ball}
\Pr\!\Big(\min_{0\le k\le N}|\beta_k(\mu)|\le \tau\Big)
\;\le\; (N+1)\,\sqrt{\frac{2}{\pi}}\;\frac{\tau}{s_{\min}},
\qquad
s_{\min}:=\min_{0\le k\le N}s_k.
\end{equation}
Equivalently, choosing $\tau=\theta\,s_{\min}/(N+1)$ with $\theta\in(0,1)$ yields
\[
\Pr\!\Big(\min_{0\le k\le N}|\beta_k(\mu)|>\tfrac{\theta}{N+1}\,s_{\min}\Big)
\;\ge\; 1-\sqrt{\tfrac{2}{\pi}}\,\theta.
\]
\end{lemma}

\begin{proof}
The Gaussian smoothing operator is the heat semigroup $W_{\mu}:=e^{(\mu/2)\partial_x^2}$, which acts lower-triangularly on the Hermite basis:
\[
W_{\mu}\,\He^{(a)}_{k+2r}
\;=\;\sum_{u=0}^{r} \frac{(k+2r)!}{(k+2r-2u)!}\,\frac{(\mu/2)^u}{u!}\;\He^{(a)}_{k+2r-2u}.
\]
Projecting onto order $k$ and noting that
$\frac{1}{k!}\E[\He^{(a)}_k(\sqrt{a}X)H_k(X)]=a^{k/2}$ and is zero for other orders, we obtain the explicit linear representation
\[
\beta_k(\mu)
\;=\; a^{k/2}\sum_{r=0}^{\lfloor (M-k)/2\rfloor}
\alpha_{k+2r}\,\frac{(k+2r)!}{k!}\,\frac{(\mu/2)^r}{r!},
\qquad 0\le k\le N.
\]
Since the $\alpha_n$ are independent $\cN(0,1)$, each $\beta_k(\mu)$ is Gaussian with mean $0$ and variance given by the sum of squares of the coefficients, which is precisely \eqref{eq:sk2}. In particular $s_k^2>0$ because the $r=0$ term contributes $a^k$.

For the small-ball bound, recall that if $Z\sim \cN(0,s^2)$, then
$\Pr(|Z|\le \tau)=\int_{-\tau}^{\tau}\frac{1}{\sqrt{2\pi}s}e^{-x^2/(2s^2)}dx
\le \frac{2\tau}{\sqrt{2\pi}s}=\sqrt{\frac{2}{\pi}}\frac{\tau}{s}$.
Applying this to each $\beta_k(\mu)\sim \cN(0,s_k^2)$ and taking a union bound over $k=0,\ldots,N$ yields \eqref{eq:small-ball}.
\end{proof}

\begin{remark}
The mapping from the Hermite coefficients $(\alpha_0,\ldots,\alpha_M)$ of $\sigma'$ to $(\beta_0(\mu),\ldots,\beta_N(\mu))$ is linear and upper-triangular in parity blocks with strictly positive diagonal; thus $s_k^2\ge a^k$ and $s_{\min}\ge 1$. The formula \eqref{eq:sk2} makes the dependence on $(\mu,a,M)$ explicit and can be used to pick a quantitative margin $\tau$ with the desired probability level.
\end{remark}

\section{Fourier--Walsh expansion of Majority.}
\label{app:majority_Fourier}
Let $\Maj_N : \{\pm 1\}^N \to \{\pm 1\}$ denote the majority function,
\[
\Maj_N(s) := \sign\!\left(\sum_{i=1}^N s_i\right),
\]
with can be broken arbitrarily (e.g., $\Maj_N(0)=1$ when $N$ is even).
The Fourier--Walsh expansion of $\Maj_N$ is given by
\[
\Maj_N(s) = \sum_{S \subseteq [N]} \hat{\Maj}_N(S)\,\chi_S(s),
\qquad
\chi_k(s) := \prod_{i \in S} s_i,
\]
where the Fourier--Walsh coefficients are defined as
\[
\hat{\Maj}_N(S) := \E_{s \sim \Unif(\{\pm 1\}^N)}\!\left[\Maj_N(s)\,\chi_S(s)\right].
\]
By symmetry, $\hat{\Maj}_N(k)$ depends only on $|S|$, and $\hat{\Maj}_N(S)=0$
whenever $|S|$ is even. For odd $|S|=k$, the coefficients admit the explicit formula
\[
\widehat{\Maj}_N(S)
=
(-1)^{\frac{k-1}{2}}
\frac{\binom{k-1}{\frac{k-1}{2}}}{2^{k-1}}
\cdot
\frac{2\,\binom{N-1}{\frac{N-1}{2}}}{2^N \binom{N-1}{k-1}},
\]

Precise asymptotics and proofs can be found in~\cite{o'donnell_2014}, Chapter 5.

\section{Experiments Details}
\label{app:experiments_details}
In this section, we provide more details about the experiments on real data shown in Section~\ref{sec:experiments}.
\subsection{Architecture and Loss Function}
\paragraph{Architecture.}
We use a two-layer fully-connected neural network with ReLU activations,
$f(x) = W_2\,\mathrm{ReLU}(W_1 x + b_1) + b_2$,
where $W_1 \in \mathbb{R}^{H \times d}$, $W_2 \in \mathbb{R}^{C \times H}$,
$H = 128$ is the hidden width, $d$ is the input dimension, and $C$ is the number
of classes.  Weights are initialized as
$W_1 \sim \mathcal{N}(0, 1/d)^{d \times H}$ and
$W_2 \sim \mathcal{N}(0, 1/H)^{H \times C}$; biases are initialized to zero.

\paragraph{Loss function.}
We minimise the cross-entropy loss with a softmax output layer.
Gradients are clipped to $[-5, 5]$ component-wise at every step.

\subsection{Datasets and Preprocessing}

\paragraph{GSE96583 (single-cell RNA-seq).}
We use batch~2 of the GSE96583 dataset \cite{kang2018multiplexed},
restricting to control (unstimulated) singlet cells.
Cells are assigned to three merged classes:
\emph{B cells}, \emph{Monocytes} (CD14$^+$ and FCGR3A$^+$ merged),
and \emph{T/NK} (CD4 T, CD8 T, and NK cells merged);
Dendritic cells and Megakaryocytes are excluded.
After filtering cells with fewer than 500 UMI counts and genes expressed
in fewer than 25 cells, we obtain $n \approx 11{,}990$ cells over
$\sim\!10{,}500$ genes.
Raw counts are library-size normalised to $10^4$ counts per cell
and log$1\!+\!x$ transformed.
For each value of $d$, the top-$d$ most variable genes are selected
by dispersion (variance/mean), using no class label information.


\subsection{Training Procedure}
We use the Adam optimiser \cite{kingma2014adam} with learning rate
$5 \times 10^{-4}$, $\beta_1 = 0.9$, $\beta_2 = 0.999$,
$\varepsilon = 10^{-8}$, batch size $256$, and $100$ epochs.

\paragraph{Train/test split and seeds.}
For each experiment, we perform a stratified $80/20$ train/test split.
The training set is then subsampled to $n$ examples (stratified by class)
for each value of $n$ considered.
All results are averaged over 5 independent random seeds; shaded bands
report $\pm 1$ standard deviation.

\end{document}